\documentclass[nohyperref]{article}

\usepackage{microtype}
\usepackage{graphicx}
\usepackage{caption}
\usepackage{subcaption}
\usepackage{booktabs} % for professional tables
\usepackage{IEEEtrantools}

\newenvironment{talign}
 {\align}
 {\endalign}
\newenvironment{talign*}
 {\csname align*\endcsname}
 {\endalign}

\usepackage{hyperref}
\hypersetup{colorlinks=false,linkcolor=false,pdfborderstyle={/S/U/W 0}} 

\usepackage[accepted]{icml2023}

\usepackage{amsmath}
\usepackage{amssymb}
\usepackage{mathtools}
\usepackage{amsthm}
\usepackage{romannum}
\AtBeginDocument{\pagenumbering{arabic}}
\usepackage[capitalize,noabbrev]{cleveref}

\theoremstyle{plain}
\newtheorem{theorem}{Theorem}[section]
\newtheorem{proposition}[theorem]{Proposition}
\newtheorem{lemma}[theorem]{Lemma}

\theoremstyle{definition}

\theoremstyle{remark}

\def\R{{\mathbb{R}}}

\def\E{{\mathbb{E}}}
\def\X{{\mathcal{X}}}
\def\P{{\mathbb{P}}}

\def\O{{\mathcal{O}}}
\DeclareMathOperator{\Tr}{Tr}
\DeclareMathOperator{\diag}{diag}
\def\DSM{{\mathcal{D}}}
\def\TDSM{{\widehat{\mathcal{D}}_m}}

\DeclareMathOperator*{\argmin}{arg\,min}
\renewcommand{\qed}{\hfill$\blacksquare$}

\definecolor{blue_plot}{RGB}{55,  126, 184}
\definecolor{orange_plot}{RGB}{255, 127, 0}
\definecolor{green_plot}{RGB}{77,  175, 74}
\definecolor{pink_plot}{RGB}{247, 129, 191}

\icmltitlerunning{Robust and Scalable Bayesian Online Changepoint Detection}

\begin{document}

\twocolumn[
\icmltitle{Robust and Scalable Bayesian Online Changepoint Detection}



\icmlsetsymbol{equal}{*}

\begin{icmlauthorlist}
\icmlauthor{Matias Altamirano}{ucl}
\icmlauthor{François-Xavier Briol}{ucl,tur}
\icmlauthor{Jeremias Knoblauch}{ucl}
\end{icmlauthorlist}

\icmlaffiliation{ucl}{Department of Statistical Science, University College London, London, United Kingdom}
\icmlaffiliation{tur}{The Alan Turing Institute, London, United Kingdom}

\icmlcorrespondingauthor{Matias Altamirano}{matias.altamirano.22@ucl.ac.uk}
\icmlcorrespondingauthor{François-Xavier Briol}{f.briol@ucl.ac.uk}
\icmlcorrespondingauthor{Jeremias Knoblauch}{j.knoblauch@ucl.ac.uk}

\icmlkeywords{changepoint detection, generalised Bayes, score-matching, robust Bayes}

\vskip 0.3in
]



\printAffiliationsAndNotice{}  % leave blank if no need to mention equal contribution

\begin{abstract}
This paper proposes an online, provably robust, and scalable Bayesian approach for changepoint detection.
The resulting algorithm has key advantages over previous work: 
it provides provable robustness by leveraging the generalised Bayesian perspective, 
and also addresses the scalability issues of previous attempts.
Specifically,
the proposed generalised Bayesian formalism leads to conjugate posteriors whose parameters are available in closed form by leveraging diffusion score matching.
The resulting algorithm is exact, 
can be updated through simple algebra, 
and is more than 10 times faster than its closest competitor.
\end{abstract}
%%%%%%%%%%%%%%%%%%%%%%%%%%%%%%%%%%%%%%%%%%%%%%%%%%%%%%%%%%%%%%%%%%%%%%%%%%%%%%%
%INTRODUCTION
%%%%%%%%%%%%%%%%%%%%%%%%%%%%%%%%%%%%%%%%%%%%%%%%%%%%%%%%%%%%%%%%%%%%%%%%%%%%%%%

\section{Introduction}

Changepoint (CP) detection is the task of identifying sudden changes in the statistical properties of a data stream. The methods to detect CPs are used in applications including systems health monitoring \citep{stival2022doubly,yang2006adaptive},  financial data \cite{kim2022unsupervised,kummerfeld2013tracking},  climate change \citep{reeves2007review, itoh2010change}, and cyber security \citep{hallgren2022changepoint}. Existing approaches include likelihood ratio methods such as the parametric method CUSUM \citep{page1954continuous} or Change Finder methods \cite{kawahara2009change}, to  Bayesian methods such as in \citet{chib1998estimation, fearnhead2006exact}. 
 
Detecting CPs in an online fashion is an even more challenging task, but can allow practitioners to act on these systems in real-time.
In a Bayesian context, the most popular method is \emph{Bayesian online changepoint detection (BOCD)} \citep{adams2007bayesian, fearnhead2007line}. 
Here, the data stream is assumed to come from one of several different underlying distributions; and the goal is to quantify our uncertainty over the most recent time at which the data distribution changed. 
%This is done using Bayesian inference to obtain posterior probabilities for the location of the most recent CP. %the amount of time since the last CP.
BOCD has many desirable properties: it is suitable for multivariate data and has the capacity to quantify uncertainty. 
However,     it also has a significant flaw inherited from Bayesian inference: it is not robust under outliers or model misspecification. 
This can lead to failures, where most data points inferred to be CPs are simply mild heterogeneities in the data.
This is a significant problem, and can causes practitioners to act on safety-critical systems based upon an erroneously declared CPs.

The lack of robustness in Bayesian methods has recently come to the forefront, and various strategies have been proposed to address it.
Arguably the most successful amongst these have been generalised Bayesian methods \citep[see e.g.][]{bissiri2016general, Jewson2018,knoblauch2019generalized}.
Building on these ideas, \citet{knoblauch2018doubly} introduced the first robust version of BOCD using generalised Bayesian inference based on $\beta$-divergences ($\beta$-BOCD). 

While the resulting algorithm is generally applicable and provides robustness, it has a major drawback that has severely impeded its broader use: it is not scalable. 
This is mainly due to the intractability of the generalised posterior and predictive distributions, which require multiple variational approximations to be performed at each time point.
As a result, $\beta$-BOCD is practically infeasible if one is interested in online methods for high-frequency data, or if one deals with a constrained computational budget.

\begin{figure}[t!]
\centering
\includegraphics[width=\columnwidth]{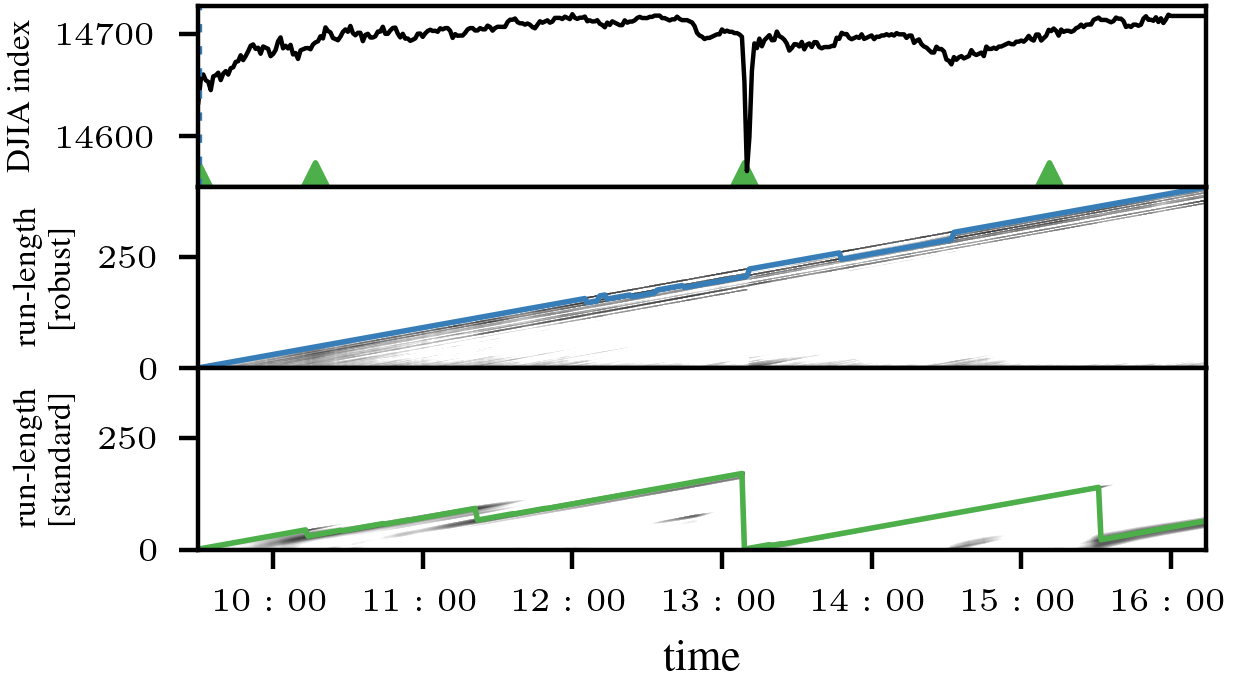}
\vspace{-0.8cm}
\caption{ \textit{Twitter Flash Crash.}
The run-length is the  time since the last changepoint (CP).
\textit{Top:} Jow Dones Index with Maximum a posteriori CPs detected by  standard BOCD marked as ${\color{green_plot}\blacktriangle}$. 
\textit{Middle \& Bottom:} run-length posteriors of $\DSM_m$-BOCD  with most likely run-length in {\textcolor{blue_plot}{\textbf{blue}}} and of standard BOCD in \textcolor{green_plot}{\textbf{green}}.
Standard BOCD incorrectly detects a CP, $\DSM_m$-BOCD does not.
}
\label{fig:flash}
\vspace{-0.7cm}
\end{figure}

This paper proposes a new generalised Bayesian inference scheme based on \emph{diffusion score matching} \citep{barp2019minimum}, which is effectively a  weighted version of the original score-matching divergence of \citet{Hyvarinen2006}.
If the weights are chosen appropriately, the resulting posteriors are provably robust, and the corresponding CP detection algorithm, denoted $\DSM_m$-BOCD, is also robust to outliers. This is illustrated in \Cref{fig:flash} on the value of the Dow Jones Industrial Average (DJIA) on the day of the `Twitter flash crash' on 17/04/2013: standard BOCD falsely identifies $3$ CPs, whilst $\DSM_m$-BOCD correctly identifies no CPs.

Additionally---and unlike posteriors based on the $\beta$-divergence---$\DSM_m$-posteriors also have a conjugacy property 
for likelihoods of the exponential family so long as the prior is chosen to be a normal, truncated normal, or any other squared exponential distribution.
This makes $\DSM_m$-BOCD very fast: specifically, it ensures that all posteriors used in the algorithm  can be updated exactly and efficiently through elementary vector and matrix calculations.
If one uses the pruning strategies for the CP posterior proposed in \citet{adams2007bayesian}, the computational complexity of our algorithm is $\mathcal{O}(T(d^2+p^2))$; where $T$ is the length of the data stream, $d$ is the dimension of the observations, and $p$ is the number of model parameters. This is the same computational complexity as the original BOCD algorithm. This also makes $\DSM_m$-BOCD more than $10$ times faster than $\beta$-BOCD in our numerical experiments.

Beyond that, $\DSM_m$-BOCD has benefits that make it more attractive than standard BOCD even from a purely computational point of view in certain settings.
For example, when modelling  $d$-dimensional observations with non-Gaussian exponential family distributions, we can obtain conjugate $\DSM_m$-posteriors,  even though no conjugate posteriors exist in the standard Bayesian case.

In summary, we make two key contributions:
\vspace*{-0.35cm}
\begin{itemize}
    \item[(1)] We derive and propose the $\DSM_m$-posterior;  \emph{proving its robustness and closed form updates} in the process; 
    \vspace*{-0.2cm}
    \item[(2)] We use this posterior for BOCD, leading to the first algorithm that is \emph{both robust and scalable}.\vspace*{-0.35cm}
\end{itemize}
The remainder of the paper is structured as follows: \cref{sec:background} reviews BOCD and generalised Bayesian inference. \cref{sec:methodology} derives the robustness and scalability properties of $\DSM_m$-posteriors, and integrates them with BOCD. We then validate our approach experimentally in \cref{sec:experiments}.

%%%%%%%%%%%%%%%%%%%%%%%%%%%%%%%%%%%%%%%%%%%%%%%%%%%%%%%%%%%%%%%%%%%%%%%%%%%%%%%
%BACKGROUND
%%%%%%%%%%%%%%%%%%%%%%%%%%%%%%%%%%%%%%%%%%%%%%%%%%%%%%%%%%%%%%%%%%%%%%%%%%%%%%%

\section{Background}
\label{sec:background}

Our method merges generalised Bayesian posteriors based on diffusion score matching with the BOCD algorithm. 
Here, we provide a short explanation of the concepts relevant for understanding this interface.

\subsection{Bayesian Online Changepoint Detection (BOCD)}
Let $x_{1:T}$ be a sequence of observations $x_1, x_2, \dots, x_T$, where $x_t \in \X \subseteq \R^{d}$ for the time index $t \in \{1,\ldots,T\}$. 
Throughout, $x_{1:T}$ follows the product partition model of \citet{barry1993bayesian}: the data is partitioned through a sequence of changepoints (CPs) $0 = \tau_1 < \tau_2 < \dots$ so that the $i$-th segment is $x_{\tau_i:\tau_{i+1}-1}$, and data within the $i$-th segment is independently and identically distributed (i.i.d.) conditional on $\tau_i, \tau_{i+1}$. 
In the model underlying BOCD, the data in each segment is modelled with the same model class $\{p_{\theta}:\theta \in \Theta\}$, but with a different parameter for each segment.
The key insight for this model, reached independently by both \citet{adams2007bayesian} and \citet{fearnhead2007line}, is that Bayesian inference can be made online and efficiently if, at time $t$, one only tracks a posterior distribution over the most recent CP.
Instead of defining a prior and posterior over the CPs directly, BOCD therefore seeks to infer the so-called \textit{run-length} $r_t$ of the current segment---the amount of time since the most recent CP.

The remainder of this section details the hierarchical Bayesian model underlying the BOCD construction. Firstly, the approach uses a conditional prior on the run-length:
\begin{talign*}
    r_{t}|r_{t-1}&\sim H(r_{t}|r_{t-1}). && \text{(Conditional prior on run-length)} 
\end{talign*}
Since at time $t$ we either have a new CP ($r_t = 0$) or the current segment continues ($r_t = r_{t-1} +1$), $H(r_t|r_{t-1})$ has positive probability mass only for $r_t \in \{0, r_{t-1}+1\}$. See \citet{Wilson2010} for a broader discussion of prior selection.
Conditional on $r_t$, all data points $x_{t'}$ from the same segment $(t-r_t):t$ so that $t' \in \{t-r_t, t-r_t+1, \dots, t\}$ are then modelled as i.i.d. from $p_{\theta}$ via
\begin{talign*}
    \theta &\sim \pi(\theta)&& \text{(Parameter prior),}
    \\
    x_{t'}|\theta &\sim p_{\theta}(x_{t'}) && \text{(Probability model for data). }
\end{talign*}
The quantity of interest is the posterior over $r_t$, which is 
\begin{talign*}
    p(r_{t}|x_{1:t}) = \dfrac{p(r_{t},x_{1:t})}{p(x_{1:t})} = \dfrac{p(r_{t},x_{1:t})}{\sum_{r_{t}=0}^{t}p(r_{t},x_{1:t})}.
\end{talign*}
This shows that the run-length posterior is tractable whenever the joint distribution between run-length and  observations given by $p(r_{t},x_{1:t})$ is also tractable.
Intriguingly, these terms can be computed efficiently via an online recursion whenever the posterior predictive is tractable:
\begin{IEEEeqnarray}{rCl}
    p(r_{t},x_{1:t}) = \hspace*{-0.2cm} \sum_{r_{t-1}= 0}^{t-1}\hspace*{-0.25cm}\underbrace{p \left(x_{t}|x^{(r_t)}_{t-1}
    %r_{t},x_{1:t-1}
    \right)}_{\text{  \:\:  Predictive Posterior}}
    \hspace*{-0.25cm}\underbrace{H(r_{t}|r_{t-1})}_{\text{CP prior}}p(r_{t-1},x_{1:t-1}),
    \nonumber
\end{IEEEeqnarray}
where $x^{(r_t)}_{t-1} = x_{t-r_t:t-1}$ is the segment with run-length $r_t$ except the most recent observation $x_t$, and the predictive of $x_t$ constructed from $x^{(r_t)}_{t-1}$ is
\begin{talign}
    p(x_{t}|x^{(r_t)}_{t-1})= \int_{\Theta}p_{\theta}(x_{t})\pi^{\operatorname{B}}(\theta|x^{(r_t)}_{t-1})%_{\text{\tiny Bayes Posterior}}
    d\theta,
    \label{eq:posterior-predictive}
\end{talign}
where $\pi^{\operatorname{B}}(\theta|x^{(r_t)}_{t-1}) \propto \prod_{i=1}^{r_t} p_\theta(x_{t-i}) \pi(\theta)$ is the Bayes posterior over $\theta$ in the current segment.
To ensure that this integral is tractable in closed form, BOCD algorithms usually use prior densities $\pi(\theta)$ and models $p_{\theta}(x)$ forming a conjugate likelihood-prior pair.

Since the standard BOCD method was proposed, it has been extended in a wide range of directions. A full literature review is beyond the scope of this paper, but we highlight extensions to Gaussian processes models \cite{saatcci2010gaussian}, non-exponential families \citep{turner2013online}, multiple models in different segments \citep{knoblauch2018spatio,knoblauch2019generalized}, observations with multiple fidelity levels \citep{Gundersen2021}, and prediction \citep{Agueldo2020}. We also note that while BOCD  only quantifies uncertainty about the most recent CP, an efficient maximum a-posteriori Viterbi-style recursion can be used to efficiently update point estimates of all CP locations \citep[see e.g.][]{fearnhead2007line}.

Unfortunately, BOCD is not robust: it finds spurious CPs whenever the model is a poor  description of data.
To address this issue, one can replace the standard Bayesian parameter posterior in \eqref{eq:posterior-predictive} with a robust generalised Bayesian posterior.

\subsection{Generalised Bayesian (GB) inference}

If the statistical model $p_{\theta}$ is well-specified so that for some $\theta_0 \in \Theta$,  the true data-generating mechanism is $p_{\theta_0}$, standard Bayesian updating is the optimal way of integrating prior information with data \citep{zellner1988optimal}.
Crucially, this no longer holds if the model is misspecified. In this setting, uncertainties are miscalibrated, posterior inferences are sensitive to outliers and heterogeneity, and the Bayesian update may no longer be the best way of processing information. To address these issues, a recent line of research has advocated for the use of generalised Bayesian inference \citep[see e.g.][]{grunwald2012safe, bissiri2016general, Jewson2018, knoblauch2019generalized, fong2021martingale, Jewson2021, matsubara2021robust} which, once conditioned on some data $x_{1:T}$, is based on a belief distribution of the form
\begin{talign}
    \pi_{\omega}^{\mathcal{D}}(\theta | x_{1:T})\propto \pi(\theta) \exp\{-\omega T \cdot \widehat{\mathcal{D}}(\theta)\}.
    \label{eq:gen-bayes}
\end{talign}
While $\widehat{\mathcal{D}}(\theta)$ could in principle represent any loss function, we consider a narrowed scope.
Specifically, for $\mathcal{D}$ being a discrepancy measure on the space of probability measures on $\mathcal{X}$,  and $p_0$ being the true data-generating process, $\widehat{\mathcal{D}}:\Theta\to\R$ uses $x_{1:T}$ to estimate the part of the discrepancy $\mathcal{D}(p_0, p_{\theta})$ that depends on $\theta$.
Here, $\omega>0$ is called the \emph{learning rate} and acts as a scaling parameter that determines how quickly the posterior learns from the data.
While the choice of $\omega$ may depend on various other considerations \citep{grunwald2012safe, holmes2017assigning}, it is typically chosen to provide approximate frequentist coverage \citep{Lyddon2019,Martin2022}.
Neither of these techniques are suitable for the online setting; and we will therefore propose a new way of choosing $\omega$ in \cref{sec:DSM-BOCD}.

The posteriors in \eqref{eq:gen-bayes} are called generalised posteriors because for $\omega =1$, and $\widehat{\mathcal{D}}(\theta) = \frac{1}{T}\sum_{t=1}^T -\log p(x_{t}|\theta)$ estimating the Kullback-Leibler divergence between the model and the data-generating process, one recovers the standard Bayes  posterior.
Using such generalisations is usually done for two main arguments: to provide robustness, and to improve computation. 
%the current paper
For example, 
\citet{Chernozhukov2003} are the first to suggest them for estimation when computing a minimum is hard.
Rather than focusing on computational aspects, \citet{Hooker2014}, \citet{Ghosh2016} and \citet{bissiri2016general} advocated for their use to improve robustness.
This has led to a flurry of papers proposing particular discrepancy measures that induce robustness \citep[e.g.][]{cherief2020mmd}, and their various applications in sequential Monte Carlo \citep{boustati2020generalised}, deep Gaussian processes \citep{knoblauch2019robust}, and Bayesian neural networks \citep{futami2018variational}.
More recently, a line of work has exploited generalised posteriors both for computational gain and robustness: \citet{matsubara2021robust,matsubara2022generalised} showcased their use for robust inference in unnormalised models with both continuous and discrete data. Similarly, 
\citet{schmon2020generalized, Dellaporta2022, Pacchiardi2021, Legramanti2022} have used them for robustness in simulation-based and likelihood-free settings.

\subsection{Generalised Bayesian Inference in BOCD}

\citet{knoblauch2018doubly} first proposed a robustification of BOCD based on \eqref{eq:gen-bayes} and the $\beta$-divergence, 
which is robust and  well-defined for  $\beta \in (0,\infty)$ when $p_{\theta}$ is uniformly bounded on $\mathcal{X}$, and whose natural estimator was derived by \citet{basu1998robust} and is given by 
\begin{talign*}
    \widehat{\mathcal{D}}_{\beta}(\theta) = \dfrac{1}{T}\sum_{t=1}^{T}\dfrac{1}{1+\beta}\int_{\X}p_{\theta}(x)^{1+\beta}dx+\dfrac{1}{\beta}p_{\theta}(x_{t})^{\beta}.
    \nonumber
\end{talign*}
While the resulting method can be made robust, it has several key failures that make it computationally infeasible in most settings. Firstly, the loss depends on  $\int_{\X}p_{\theta}(x)^{1+\beta}dx$. Unless this integral is available in closed form, using $\widehat{\mathcal{D}}_{\beta}$ will introduce the same challenges as working with an intractable likelihood in a standard Bayesian setting.
Secondly, the hyperparameter $\beta$ enters the loss as the  exponent of a likelihood. Numerically, this makes the loss extremely sensitive to even very minor changes in $\beta$, which makes it very difficult to tune $\beta$ and counteracts the very robustness one hopes to achieve.
This numerical instability is compounded by the fact that \eqref{eq:gen-bayes} depends on the exponentiation of $\widehat{\mathcal{D}}_{\beta}$---if $p_{\theta}$ is an exponential family member, then even if one ignores the integral term, $\exp\{ -\omega T \widehat{\mathcal{D}}_{\beta}(\theta) \}$ is a double exponential. 
Thirdly, posteriors based on $\widehat{\mathcal{D}}_{\beta}$ often have to be approximated using variational methods.
Since this has to be done for all run-lengths $r_t$ at each time step $t$  for the recursive relationship powering the algorithm, this represents a substantive computational overhead.

Taken together, these issues often render posteriors based on the $\beta$-divergence computationally infeasible; especially in high-dimensional settings. 
In principle, one could replace the $\beta$-divergence with various robust alternatives whose numerical issues are less substantive and whose hyperparameters are easier to tune---such as $\alpha$-divergences \citep{Hooker2014}, $\gamma$-divergences \citep{knoblauch2019robust}, or maximum mean discrepancies \citep{cherief2020mmd}.
Unfortunately, none of these alternatives alleviate the problem of computationally expensive variational approximations.
This is an issue, since ultimately, it is the conjugate forms that can be updated in terms of sufficient statistics that make BOCD computationally attractive.

In the face of this, it may be tempting to postulate an inherent trade-off between robustness and computational tractability for generalised Bayes.
But this is not so; recently, it was shown that robust posteriors based on kernel Stein discrepancies have a conjugacy property \citep[Proposition 2 of][]{matsubara2021robust}.
These generalised posteriors however are not suitable for BOCD: Updating them from $t-1$ to $t$ observations takes $\mathcal{O}(t)$ operations---as opposed to the $\mathcal{O}(1)$ operations required for standard Bayesian posteriors. 
Such updates would lead to an  algorithm whose computational demands per iteration increase linearly the longer it is run, leading to an `online' algorithm in name only. 
This is why the current paper proposes a new class of generalised posteriors based on diffusion score matching \citep{barp2019minimum}: we prove that they are robust, and lead to conjugacy, with closed forms updates that take $\mathcal{O}(1)$ operations.

%%%%%%%%%%%%%%%%%%%%%%%%%%%%%%%%%%%%%%%%%%%%%%%%%%%%%%%%%%%%%%%%%%%%%%%%%%%%%%%
%METHODOLOGY
%%%%%%%%%%%%%%%%%%%%%%%%%%%%%%%%%%%%%%%%%%%%%%%%%%%%%%%%%%%%%%%%%%%%%%%%%%%%%%%

\section{Methodology}
\label{sec:methodology}

We present the methodological innovations of the current paper in three steps: After an exposition of diffusion score matching, we first explain how the resulting generalised Bayesian posterior yields closed form updates.
In a second step, we provide formal robustness guarantees for these posteriors.
In the last step, we show how to integrate them into the BOCD framework, yielding $\DSM_{m}$-BOCD; and how to choose its hyperparameters.

\subsection{Diffusion Score Matching Bayes}
\label{sec:DSM-Bayes}

\subparagraph{Notation.}
We write the divergence operator on a vector field $f$ as $\nabla \cdot f$. 
This condenses the formulae derived in this paper, but we provide all uncondensed versions in  \cref{appendix:background}.
The $d$-dimensional vector (and $d\times p$ sized matrix) of partial derivatives for $f:\mathcal{X} \to \mathbb{R}$ (and $g:\mathcal{X} \to \mathbb{R}^p$) evaluated at $x \in \mathcal{X}$ is written as $\nabla f(x)$ (and $\nabla g(x)$).

\vspace{-2mm}

\subparagraph{Score Matching.} 
Score matching is a discrepancy-based method for estimating parameters first proposed by \citet{Hyvarinen2006}.
The key idea is to approximately minimise the Fisher divergence between the statistical model $\{p_{\theta}:\theta \in \Theta\}$ and the data-generating process $p_0$.
This method takes its name from the fact that for a density $p$ on $\mathcal{X}$ and $s_{p}(x) = \nabla \log p(x)$---the so-called \textit{score function} of the density $p$---the Fisher divergence is 
\begin{talign*}
    \DSM_{I_d}(p_0||p_\theta) &= \E_{X\sim p_0 } \left[\|s_{p_{\theta}}(X) - s_{p_0}(X)\|_{2}^{2}\right].
    %\nonumber
\end{talign*}
This divergence is therefore minimised by matching the scores of the model to that of the data-generating process $p_0$.
This objective is convenient for two main reasons: Firstly, for the  density $p = \tilde{p}\frac{1}{Z}$ with normaliser $Z>0$, $s_p = s_{\tilde{p}}$, so that the objective is attractive when working with likelihoods whose normaliser $Z$ is unknown.
Secondly, the objective can be rewritten so that the scores of $p_0$ do not have to be estimated to compute it.

Score matching has been used widely, including for data on manifolds or other complex domains \citep{mardia2016score,liu2022estimating,Scealy2022}, energy-based models \citep{Vincent2011}, anomaly detection \citep{Zhai2016}, nonparametric density estimation \citep{Sriperumbudur2017}, score-based generative modelling \citep{Song2019}, and even for Bayesian model selection \citep{Dawid2015,Shao2019,Jewson2021} or as a scoring rule \citep{Parry2012}.
In recent work, \citet{wu2023quickest} used score matching for change point detection.
This work differs from ours in three major ways: they consider a frequentist setting based on the CUSUM statistic, they only consider standard score matching, and they are not concerned with robustness.
Building on these successes, various generalised forms of score matching have been proposed over the years to address some of its shortcomings \citep[e.g.][]{Lyu2009,xu2022generalized,Yu2022, matsubara2022generalised}.

\vspace{-2mm}

\paragraph{Diffusion Score Matching.} The particular generalisation we consider hereafter is \emph{diffusion score matching}, which was introduced in \citet{barp2019minimum} and amounts to a weighted version of the Fisher divergence given as
\begin{talign*}
    \DSM_{m}(p_0\|p_{\theta}) = \E_{X\sim p_0 } \left[\|m^{\top}(X)(s_{p_{\theta}}(X) - s_{p_0}(X))\|_{2}^{2}\right],
    %\nonumber
\end{talign*}
for a pointwise invertible matrix-valued function $m:\mathcal{X} \to \mathbb{R}^{d\times d}$. The function $m$ is also known as diffusion matrix due to the construction of this distance as a Stein discrepancy with a pre-conditioned diffusion Stein operator; see \citet{Anastasiou2021} for full details.

Like $\DSM_{I_d}$, $\DSM_m$ is a statistical divergence between densities $p_{0}$ and $p_{\theta}$ on $\X = \R^d$ whenever $\int_{\X}| s_{p_{\theta}}(x) - s_{p_0}(x)|^2p_0(x)dx < \infty$. 
%
%$s_{p_{\theta}} - s_{q} \in L^{2}(q)$
Under appropriate smoothness and boundary conditions, this can be extended to the case where $\mathcal{X}$  is a connected subset of $\mathbb{R}^d$ \citep{liu2022estimating, Zhang2022}.
More generally, $\DSM_m$ recovers $\DSM_{I_d}$ for $m(x) = I_d$ (the $d-$dimensional identity matrix), the estimator in \citet{Hyvarinen2007} for $m(x) = x$, and the generalised h-score matching method for $m(x) = \operatorname{diag}(h^{1/2}(x))$, where $h$ is defined in \citet{Yu2018, Yu2019}.
The function $m$ can be thought of as up-weighting areas of $\mathcal{X}$ on which matching the scores of the model to that of the data-generating process is most important.
For the purposes of the current paper, we will choose this weight to ensure that the constructed generalised posteriors are provably robust (see \cref{sec:robustness} for details).

Estimating $\DSM_m$ directly is challenging, as it would require estimating the unknown score $s_{p_0}$.
Fortunately, under the aforementioned smoothness and boundary conditions (\cref{appendix:boundary}) \citep{liu2022estimating}, we  can expand the above equation and use integration by parts.
Then, up to a constant that does not depend on $\theta$, we can rewrite $\DSM_{m}(p_0\|p_{\theta})$ as
\begin{IEEEeqnarray}{rCl}
    \E_{X\sim p_0 } [\|(m^{\top}s_{p_{\theta}})(X)\|_{2}^{2} 
    +(2\nabla\cdot(mm^{\top} s_{p_{\theta}}))(X) ]. \quad  
    \label{eq:DSM-expansion}
\end{IEEEeqnarray}
Crucially, the quantity above no longer features $s_{p_0}$, and only depends on $p_0$ through an expectation.
This leads to a natural estimator which for $x_{1:T}$ is given by 
\begin{talign*}
   \widehat{\mathcal{D}}_m(\theta) & =  \dfrac{1}{T}\sum_{t=1}^{T}d_m(\theta, x_t), \quad \text{ where } \nonumber \\
    d_m(\theta, x_t) & =  \|(m^{\top}s_{p_{\theta}})(x_{t})\|_{2}^{2}  +(2\nabla\cdot(mm^{\top} s_{p_{\theta}}))(x_{t}).
\end{talign*}

\vspace{-2mm}

\paragraph{Diffusion Score Matching Bayes.} Based on the estimator $\widehat{\mathcal{D}}_m$ for the part of $\mathcal{D}_m$ that depends on $\theta$, we can construct 
\begin{IEEEeqnarray}{rCl}
    \pi^{\DSM_m}_{\omega}(\theta| x_{1:T}) \propto \pi(\theta) \exp(-\omega T \widehat{\mathcal{D}}_{m}(\theta)).  
    \label{eq:DSM-bayes}
\end{IEEEeqnarray}
Using score matching for a generalised Bayes posterior was first discussed in passing in Section 4.2 of \citet{Giummole2019}, though the context is about reference priors for objective Bayesian inference, and the method is only briefly mentioned. 
This previous work also does not robustify the resulting posterior through the introduction of a weighting matrix $m$, or derive its conjugate posteriors.

\subsection{Conjugacy for Exponential Family Models}
\label{sec:conjugacy}
The conjugacy of posteriors of the form \eqref{eq:DSM-bayes} make them more attractive than potential alternatives. For exponential family likelihoods, these posteriors  depend on two parameters available in closed form.
The exponential family is given the collection of models with a probability density function 
\begin{IEEEeqnarray}{rCl}
    p_{\theta}(x) = \exp{(\eta(\theta)^{\top}r(x)-a(\theta)+b(x))},
    \label{eq:exponential-family}
\end{IEEEeqnarray}
where $\eta:\Theta\to\R^{p}$, $r:\X\to\R^{p}$,  $a:\Theta\to\R$, and $b:\X\to\R$. When $\eta(\theta) = \theta$, we say that the exponential family model is in natural form, and one can reparametrise a model to natural form by reparameterising with the map $\eta^{-1}$.
Exponential family class of distributions includes the Gaussian, exponential, Gamma, and Beta distributions.
\begin{proposition}
\label{DSM-exponential}
If $p_{\theta}$ is given by \eqref{eq:exponential-family}, then
\begin{IEEEeqnarray*}{rCl}
    \pi^{\DSM_m}_{\omega}(\theta| x_{1:T}) \propto \pi(\theta) \exp(-\omega T [\eta(\theta)^{\top} \Lambda_{T}\eta (\theta)+\eta (\theta)^{\top}\nu_{T}]),
\end{IEEEeqnarray*}
for $\Lambda_{T} = \frac{1}{T}\sum_{t=1}^{T}\Lambda(x_{t})$, $\nu_{T} = \frac{2}{T}\sum_{t=1}^{T} \nu(x_{t})$, and 
\begin{talign*}
    \Lambda(x) &= (\nabla r^{\top}mm^{\top} \nabla r)(x),\\
    \nu(x) &= \left( \nabla r^{\top}mm^{\top} \nabla b + \nabla\cdot(mm^{\top}\nabla r)\right)(x).
\end{talign*}
Taking  $\eta(\theta)=\theta$ and choosing a squared exponential prior $\pi(\theta) \propto\exp{(-\frac{1}{2} (\theta-\mu)^{\top}\Sigma^{-1}(\theta-\mu))}$, also makes $\pi^{\DSM_m}_{\omega}(\theta| x_{1:T})$ a (truncated) normal of the form
\begin{talign*}
   \pi^{\DSM_m}_{\omega}(\theta| x_{1:T}) &\propto\exp{\left(-\frac{1}{2} (\theta-\mu_{T})^{\top}\Sigma_{T}^{-1}(\theta-\mu_{T})\right)},
\end{talign*}
    for $\Sigma_{T}^{-1} = \Sigma^{-1}+2\omega T \Lambda_{T}$ and
    $\mu_{T} = \Sigma_{T} \left(\Sigma^{-1}\mu-\omega T \nu_{T}\right)$.
\end{proposition}
The proof is in \cref{proof:DSM-exponential}. The natural exponential family allows us to recover a form of Gaussian conjugacy, since the diffusion score matching squared becomes a quadratic form in this case. This renders DSM-Bayes scalable; as we will elaborate upon in \cref{sec:DSM-BOCD}, $\Sigma_{T}^{-1}$ and $ \mu_{T}$ can be updated with a new observation in $\O(p^2+d^2)$ operations.

\subsection{Global Bias-Robustness}
\label{sec:robustness}
\begin{figure}[t]
    \centering
\includegraphics{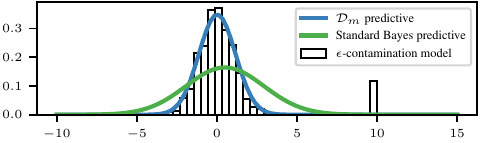}
    \vspace*{-0.4cm}
    \caption{
    \textit{Impact of misspecification in posteriors.}
    The robust \textcolor{blue_plot}{\textbf{$\DSM_m$-posterior}} and non-robust \textcolor{green_plot}{\textbf{standard Bayes}} posterior predictive when the data are incorrectly modelled as Gaussian, but  follow an $\varepsilon$-contamination model $\P = 0.95\mathcal{N}(0,1)+0.05\delta_{10}$.
    }
    \label{fig:contamination}
    \vspace*{-0.4cm}
\end{figure}

Building a BOCD algorithm based on $\pi^{\DSM_m}_{\omega}(\theta| x_{1:T})$ is attractive not only computationally, but also due to its robustness.
We prove this robustness formally by using the classical framework of $\varepsilon$-contamination models \citep[see, e.g.][]{huber2011robust}.
Given a distribution $\P$, we consider its $\varepsilon$-contaminated counterpart $\P_{\varepsilon,y} = (1-\varepsilon)\P+\varepsilon\delta_{y}$, where  $\delta_y$ is the dirac-measure at some $y\in\X$, and $\varepsilon\in[0,1]$. 
The classical perspective on robustness  proceeds by defining a point estimator $E:\mathcal{P}(\mathcal{X}) \to \Theta$ that maps from $\mathcal{P}(\mathcal{X})$, the space of distributions on $\mathcal{X}$, to $\Theta$. 
One then investigates its robustness via $\lim_{\varepsilon\to 0}\frac{1}{\varepsilon}\|E(\P)-E(\P_{\varepsilon, y})\|_2$, which under mild conditions is equivalent to the derivative  $\frac{\partial}{\partial\varepsilon}\|E(\P_{\varepsilon, y})\|_2\big|_{\varepsilon=0}$.
This limit is the so-called \textit{influence function}. It quantifies the impact of an infinitesimal contamination at $y$ on the estimator, and is a classical tool to measure outlier robustness.

The Bayesian case is slightly more complicated and depicted in \cref{fig:contamination}: we are not concerned by estimators on $\Theta$, but on $\mathcal{P}(\Theta)$.
The estimates under study are thus  infinite-dimensional objects that vary over $\Theta$.
To get a handle on this, we first define an influence function \textit{pointwise} for each $\theta \in \Theta$.
To this end, note that $\widehat{\mathcal{D}}_{m}(\theta) = \mathbb{E}_{X \sim \P_T}[d_m(\theta, X)]$. 
We can now define the density-valued estimator $\pi^{\DSM_m}_{\omega}(\theta|\P) \propto \pi(\theta) \exp\{ - \omega T \mathbb{E}_{X \sim \P}[d_m(\theta, X)] \}$, noting 
$\pi^{\DSM_m}_{\omega}(\theta|\P_T) = \pi^{\DSM_m}_{\omega}(\theta|x_{1:T})$ for $\P_T = \frac{1}{T}\sum_{t=1}^T\delta_{x_t}$.
Its pointwise posterior influence function (PIF) is 
\begin{IEEEeqnarray}{rCl}
    \text{PIF}(y,\theta,\P) = \dfrac{d}{d\varepsilon}\pi^{\DSM_m}_{\omega}(\theta|\P_{\varepsilon, y})\big|_{\varepsilon=0}.
    \nonumber
\end{IEEEeqnarray}
Since this is a definition of sensitivity that is local to both $\theta$ and $y$, making a global statement for all of $\pi^{\DSM_m}_{\omega}(\theta|x_{1:T})$ requires that we aggregate a notion of sensitivity over both arguments.
The easiest way to do this is to investigate $\sup_{\theta \in \Theta, y \in \mathcal{X}}\text{PIF}(y,\theta,\P_{T})$.
If this double supremum is bounded, we call a posterior \textit{globally bias-robust}, which means that the impact of contamination on the posterior density is uniformly bounded---both over the parameter space, and the location of said contamination in the data space.
This way of studying the robustness of generalised posteriors was pioneered in \citet{Ghosh2016}, and extended by \citet{matsubara2021robust}.
We build on these advances, and provide a simple condition on $m$ for global bias-robustness of $\pi^{\DSM_m}_{\omega}(\theta|x_{1:T})$ in some exponential family models.
\begin{proposition}
\label{Robust m}
If $p_{\theta}$ is as in \eqref{eq:exponential-family} so that $ \eta(\theta) = \theta $ and $\nabla b = 0$, and if the prior is a squared exponential as in \Cref{DSM-exponential}, then $\pi^{\DSM_m}_{\omega}(\theta|x_{1:T})$ is globally bias-robust if  $m:\X\to\R^{d\times d}$ is chosen so  that $\theta^{\star}\neq 0_p$  and
\begin{talign*}
 m_{ij}(x) = \left\{
 \begin{aligned}
     &\dfrac{1}{\sqrt{1+(\nabla r(x) \theta^{\star})_{i}^{2}}} && \text{if }i=j, \\
     &0 && \text{if }i\neq j.
 \end{aligned}
\right.
\end{talign*}
\end{proposition}
While $m$ could in principle depend on $\theta$, this would break the conjugacy presented in \Cref{DSM-exponential}.
The above choice of $m$ does \textit{not} depend on $\theta$, and therefore maintains the computational advantages of $\pi^{\DSM_m}_{\omega}(\theta|x_{1:T})$.
The result's conditions are also mild:
we can always ensure that $\eta(\theta) = \theta $ by re-parameterising.
Similarly, most distributions of interest satisfy $\nabla b = 0$. Examples include Gaussians, exponentials, (inverse) Gamma, and Beta distributions.
Note also that $m$ is only applicable to models with support $\mathcal{X}=\R^{d}$, as the expansion in \eqref{eq:DSM-expansion} is otherwise not valid without additional boundary conditions.
However, we prove that the proposed weight matrix $m$ 
 also leads to a well-defined discrepancy measure for various distributions defined on subsets of $\mathcal{X}$, including the Gamma and the exponential distribution (see \cref{appendix:boundary}).

\subsection{$\DSM_m$-BOCD}
\label{sec:DSM-BOCD}

Using our robust posterior within BOCD is straightforward, as its only appearance is in the  posterior predictive via
\begin{talign*}
    p\big(x_{t}|x^{(r)}_{t-1}\big) = \int_{\Theta}p_{\theta}(x_{t})\pi^{\DSM_m}_{\omega}\big(\theta|x^{(r)}_{t-1}\big)d\theta.
    \nonumber
\end{talign*}
If $p_{\theta}$ is a natural exponential family with a squared exponential prior, then $\pi^{\DSM_m}_{\omega}(\theta| x^{(r)}_{t-1}))$ is a normal distribution parameterised by inverse covariance matrix $\Sigma_{t-1, r}^{-1}$ and mean $\mu_{t-1, r}$ by virtue of \ref{DSM-exponential}.
This makes the predictive  easy to compute---either in closed form or by sampling from $\pi^{\DSM_m}_{\omega}$---which is a significant advantage over the $\beta$-BOCD framework. For the latter, the posterior will generally be intractable so that the algorithm relies on variational approximations.
Importantly, there is no way to both efficiently and exactly update variational approximations based on $x_{1:t}$ once observation $x_{t+1}$ arrives: one either uses cheap updates that lead to subpar variational approximations of the posterior, or one re-computes the approximation from scratch  at the expense of a substantive computational overhead.

In contrast, our approach allows for a cheap and exact update: if we store
$\Sigma_{t-1, r}^{-1}$ and $\mu_{t-1, r}$, we can perform the update $\pi^{\DSM_m}_{\omega}(\theta|x^{(r)}_{t-1}) \mapsto \pi^{\DSM_m}_{\omega}(\theta|x^{(r+1)}_{t})$ that adds $x_t$ into the parameter posterior of the segment $x^{(r)}_{t-1}$ via
\begin{talign*}
    \Sigma_{t, r+1}^{-1} &= \Sigma_{t-1, r}^{-1}+2\omega \Lambda(x_{t}),
    \nonumber \\
    \mu_{t, r+1} &=   \Sigma_{t, r+1} \left(\Sigma_{t-1, r}^{-1}\mu_{t-1, r} -2\omega\nu(x_{t})\right).
\end{talign*}
If we have access to the un-inverted matrix $\Sigma_{t, r+1}$, all of these operations are basic matrix and vector additions or multiplications that take $\O(p^2+d^2)$ operations to execute.
While naively computing $\Sigma_{t, r+1}$ from $\Sigma_{t, r+1}^{-1}$ would take $\O(p^3)$ operations, we can apply the Sherman-Morrison formula to the update of $\Sigma_{t, r+1}^{-1}$ to reduce this to $\O(p^2)$, maintaining the overall complexity of $\O(p^2+d^2)$. 
This is also the complexity of standard BOCD with the Gaussian likelihood and conjugate prior \citep{adams2007bayesian}.
In CP methods for high-frequency data, both the number of parameters $p$ and the data dimension $d$ are typically small, so that an update of $\O(p^2+d^2)$ is attractive.

\vspace{-2mm}

\paragraph{Run-length pruning.}
A naive implementation of $\DSM_m$-BOCD would keep a posterior over all possible run-lengths $r_t = \{0,1\dots, t-1\}$, but this would lead to an algorithm with overall complexity $\O(\sum_{t=1}^T t(d^2+p^2) ) = \O(T^2(d^2+p^2))$ for a time series of length $T$.
To prevent this, authors have proposed to `prune' the run-length posterior to a constant length \citep{adams2007bayesian, fearnhead2007line}.
Here, we follow the most popular strategy \citep[e.g.][]{adams2007bayesian, saatcci2010gaussian, knoblauch2018spatio} by keeping only the $k$ most probable run-lengths. 
For all experiments, we take $k=50$.

\begin{figure*}[t]
    \centering    \includegraphics[trim= {0cm 0.4cm 0cm 0cm}, clip]{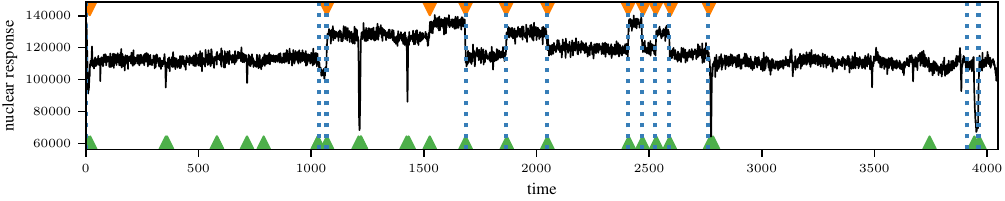}
    \vspace*{-0.4cm}
    \caption{ \textit{Well-log data.} MAP segmentation  indicated by {\textcolor{blue_plot}{\textbf{blue}}} dashed lines for $\DSM_m$-BOCD, ${\color{orange_plot}\blacktriangledown}$ for $\beta$-BOCD, and  ${\color{green_plot}\blacktriangle}$ for standard BOCD.
    Standard BOCD mistakenly labels outliers as CPs, while both $\DSM_m$-BOCD and $\beta$-BOCD are robust and identify lasting changes.
    }
    \label{fig:well}
    \vspace*{-0.4cm}
\end{figure*}

\vspace{-2mm}

\paragraph{Choice of $m$} Throughout, we choose $m$ as per \Cref{Robust m}, as it ensures robustness---even for certain distributions with boundaries (see \cref{appendix:boundary}).
Regarding $\theta^{\star}$, we found that $\DSM_m$-BOCD was not very sensitive to this choice; likely because tuning $\omega$ offsets any sensitivity to it. 
In all experiments, we thus picked $\theta^{\star}$ as the maximum likelihood estimate computed on the full data set. We note that one known issue with robust CP detection method is that they can experience a latency when it comes to detecting actual CPs. Interestingly, this is not something we observe in our experiments with this choice of $m$ and $\theta^*$.

\vspace{-2mm}

\paragraph{Choice of $\omega$}
How to choose $\omega$ is an important question for generalised Bayesian inference, and has more than one answer \citep{Lyddon2019,Syring2019,matsubara2022generalised,Bochkina2022,Wu2023}. 
Previous methods are computationally expensive, asymptotically motivated, and focus on tuning the learning rate to provide asymptotically correct frequentist coverage.
As the computational overhead of these methods is substantial and their asymptotic arguments generally do not apply to the CP setting, we pursue a different  strategy: 
we match the uncertainty of the generalised posterior to that of its standard counterpart on the first $t^{\star}$ observations of the data stream.
To operationalise this, we choose
\begin{talign*}
    \omega^{\star} = \argmin_{\omega>0} \operatorname{KL} \left(
        \pi^{\DSM_m}_{\omega}
        (\theta|x_{1:t^{\star}}) 
        \| \pi^{\operatorname{B}}
        (\theta|x_{1:t^{\star}})
    \right). 
    %\nonumber
\end{talign*}
Computing $\omega^{\star}$ is implemented using automatic differentiation via \texttt{jax} \citep{jax2018github}. This is possible  even if the standard Bayes posterior $\pi^{\operatorname{B}}$  is intractable, since $\pi^{\DSM_m}_{\omega}$ has a conjugacy property (see \Cref{DSM-exponential}).
Since the standard Bayes posterior is reliable in the absence of outliers and heterogeneity, this yields reasonable uncertainty quantification if the degree of misspecification is mild at the beginning of the data stream. 
Our experiments confirm this: the uncertainty is well-calibrated, both predictively and with regards to the run-length posterior.

 \begin{figure}[t!]
     \centering
     \includegraphics[trim= {0cm 0.4cm 0cm 0cm}, clip]{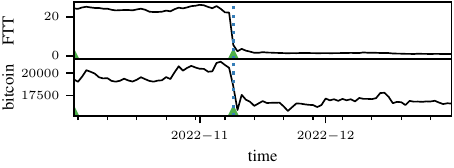}
     \vspace*{-0.4cm}
      \caption{
      \textit{Crypto-crash.}
      MAP segmentation  indicated by {\textcolor{blue_plot}{\textbf{blue}}} dashed lines for $\DSM_m$-BOCD, and  by ${\color{green_plot}\blacktriangle}$ for standard BOCD.
      There are no outliers, so both methods  identify the correct CP.
      }
     \label{fig:tfx}
     \vspace*{-0.4cm}
 \end{figure}

%%%%%%%%%%%%%%%%%%%%%%%%%%%%%%%%%%%%%%%%%%%%%%%%%%%%%%%%%%%%%%%%%%%%%%%%%%%%%%%
%EXPERIMENTS
%%%%%%%%%%%%%%%%%%%%%%%%%%%%%%%%%%%%%%%%%%%%%%%%%%%%%%%%%%%%%%%%%%%%%%%%%%%%%%%

\section{Experiments}
\label{sec:experiments}

We investigate $\DSM_m$-BOCD empirically in several numerical experiments. 
In doing so, we highlight its computational and inferential advantages over standard BOCD and $\beta$-BOCD.
In all experiments, we choose conjugate priors as in \cref{DSM-exponential}, and $m$ and $\omega$ as in \cref{sec:DSM-BOCD}. All code and data is publicly available at \url{https://github.com/maltamiranomontero/DSM-bocd}.

\vspace{-3mm}

\paragraph{Computational complexity.}

We compare the complexity of the three BOCD methods in different settings and show that $\DSM_m$-BOCD is considerably faster than $\beta$-BOCD, even when sampling is needed. 
Moreover, \cref{fig:complexity_mean_T} shows that $\DSM_m$ is as fast as standard BOCD when $d=1$ and the predictive posterior is available in closed form. See \cref{app:additional-computation-experiemnt} for details. 

\begin{figure}[ht]
  \centering
  \includegraphics[width=\linewidth]{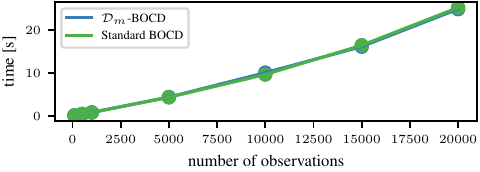}
  \vspace{-8mm}
  \caption{Overall time in seconds versus number of observations. We observe that both methods are equally fast for any number of observations.}
  \label{fig:complexity_mean_T}
\end{figure}

\vspace{-3mm}

\paragraph{Accuracy and detection delay.}

We quantify the method's performance advantage by comparing detection delay and accuracy on artificially generated data with outliers. We generate 600 samples with 2\% of outliers and 6 CPs; then, we report the positive predictive value (PPV), true positive rate (TPR), and detection delay. See \cref{app:accuracy} for the exact expression of the metrics.
\cref{tab:accuracy} shows that our method detects the same amount of true positives as the standard BOCD while not detecting many false positives, showing the strength of our method. Moreover, the detection delay shows that in spite of being robust to outliers, $\DSM_m$-BOCD does not cause any delay in the detection of CP.

\begin{table}[h!]
\centering
\resizebox{\columnwidth}{!}{
\begin{tabular}{@{}llll@{}}
\toprule
Method        & PPV            & TPR            & Delays        \\ \midrule
$\DSM_m$-BOCD      & \textbf{0.907$\pm$0.154} & \textbf{0.883$\pm$0.13} & 1.643$\pm$0.475         \\
Standard BOCD & 0.6$\pm$0.128            & 0.833$\pm$0.149          & \textbf{1.05$\pm$1.545} \\ \bottomrule
\end{tabular}}
\caption{Performance indices-mean and standard deviation-using the positive predictive value (PPV), true positive rate (TPR), and the detection delays over 10 realisations. For PPV and TPR, the nearest to 1, the better. For delays, the lower, the better.}
\label{tab:accuracy}
\end{table}

\begin{figure*}[ht]
    \centering
    \includegraphics[trim= {0cm 0.35cm 0cm 0cm}, clip, width=\textwidth]{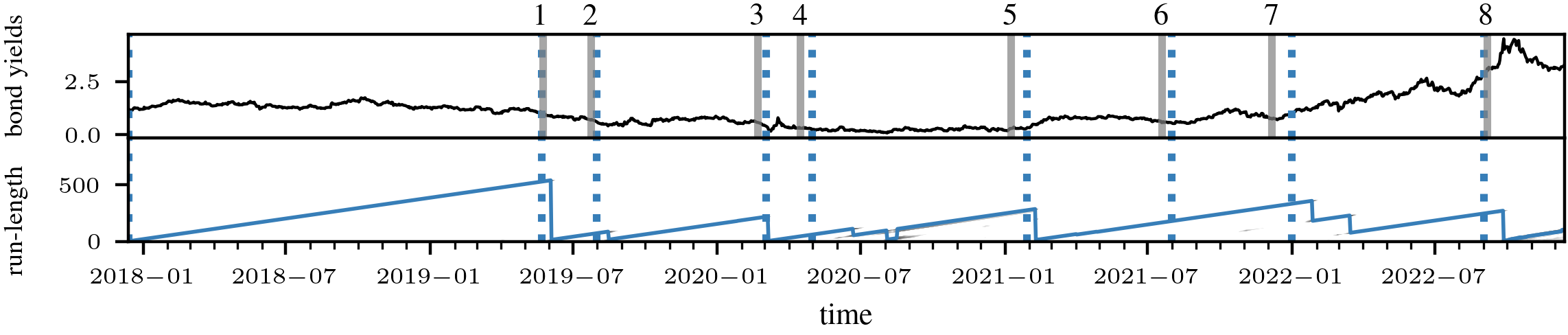}
    \vspace*{-0.8cm}
    \caption{
    \textit{UK's 10 year government bond yield 2018-2023.}
    The MAP-segmentation resulting from $\DSM_m$-BOCD is indicated in dashed {\textcolor{blue_plot}{\textbf{blue}}} lines.
    The bottom panel displays the corresponding run-length posterior, with the most likely run-length marked in {\textcolor{blue_plot}{\textbf{blue}}}.
    A series of political events of national importance closely track the segmentation, and are marked with solid
    \textcolor{gray}{\textbf{gray}} lines:
    1. Theresa May announces her resignation from her position as prime minister;
    2. Boris Johnson sworn in as prime minister;
    3. the first Covid case recorded in EU;
    4. the first Covid wave in the UK is officially declared;
    5. the third Covid wave in the UK is officially declared;
    6. the legal limits on social contact removed in UK;
    7. Covid 'Plan B' measures are implemented in UK in response to the spread of the Omicron variant;
    8. Liz Truss is sworn in as prime minister.}
    \vspace*{-0.4cm}
    \label{fig:bond}
\end{figure*}

\vspace{-3mm}

\paragraph{Twitter flash crash \& Cryptocrash.}

A robust CP detection algorithm must not to be fooled by outliers while detecting CP correctly. 
We show that $\DSM_m$-BOCD has this capability on two real-world examples: the first is the Dow Jones Industrial Average (DJIA) index every minute on 17/04/2013, the day of the \textit{Twitter flash crash}. The data is publicly available on FirstRate Data.\footnote{\href{https://firstratedata.com/free-intraday-data}{https://firstratedata.com/free-intraday-data}}
That day, the Associated Press' Twitter account was hacked and falsely tweeted that explosions at the White House had injured then-president Barack Obama.
In response, the DJIA dropped by 150 points in a matter of seconds before bouncing back. 
As \cref{fig:flash} shows, this is a clear outlier. Modelling the time series with a Gaussian, the plot shows that $\DSM_m$-BOCD successfully ignores this blip, while standard BOCD incorrectly labels it as a CP. 
The second example tracks the average daily value of FTT and Bitcoin between 10/2022 and 12/2022, data which is publicly available on Yahoo finance.\footnote{\href{https://finance.yahoo.com/}{https://finance.yahoo.com/}} FTT was the token issued by FTX, one of the biggest crypto-exchanges before it failed due to a liquidity crisis on November 11th 2022. 
The ensuing collapse of FTX marked a crash in the value of various crypto-currencies, including Bitcoin.
Using a two-dimensional Gaussian distribution for both $\DSM_m$-BOCD and standard BOCD, \cref{fig:tfx} shows  that both methods correctly detect the CP. 
\cref{fig:tfx_full} in \cref{appendix:expDetails} also displays the run-length posteriors, and shows that robustness does not lead to increased CP detection latency.
\vspace{-3mm}

\paragraph{Well-log.}
The well-log data was introduced in \citet{ruanaidh1996numerical}, and consists in 4,050 nuclear magnetic resonance measurements recorded while drilling a well. CPs in the sequence correspond to changes in the sediment layers the drill is penetrating.
On top of these clear changes, the data contains outliers and contaminants corresponding to more short-term events in geological history---such as flooding, earthquakes, or volcanic activity.
When this data set is studied, its outliers have traditionally
been removed before CP detection algorithms are run \citep[see e.g.][]{adams2007bayesian, BSCPD2, LassoCP}.
We leave them in, and \cref{fig:well} shows that this is unproblematic for $\DSM_m$-BOCD, but does lead to falsely labelled CPs with BOCD.
We also compare the algorithm with  $\beta$-BOCD   \citep{knoblauch2018doubly}, and find that the detected changes are almost identical.
On a machine with processor Intel i7-7500U 2.7 GHz, and 12GB of RAM, $\DSM_m$-BOCD took about 10 times less than $\beta$-BOCD.

\vspace{-3mm}

\paragraph{Multivariate synthetic data.}

In certain settings, $\DSM_m$-posteriors are conjugate when standard posteriors are not.
An example is a multivariate time series whose dimensions follow different distributions belonging to the exponential family.
To this end, we generate 1000 samples from a time series with CPs at $t=250, 750$.
Conditional on the CPs, the data is generated independently from an exponential in the first dimension and Gaussian distribution in the second dimension.
$\DSM_m$-BOCD is immediately applicable, and \cref{fig:synthetic} shows that the algorithm functions reliably.
We do not compare to BOCD in this setting: for this model, standard Bayesian posteriors would require expensive sampling algorithms or variational approximations to be employed, rendering the algorithm impractical.

\begin{figure}[t]
    \centering
    \includegraphics{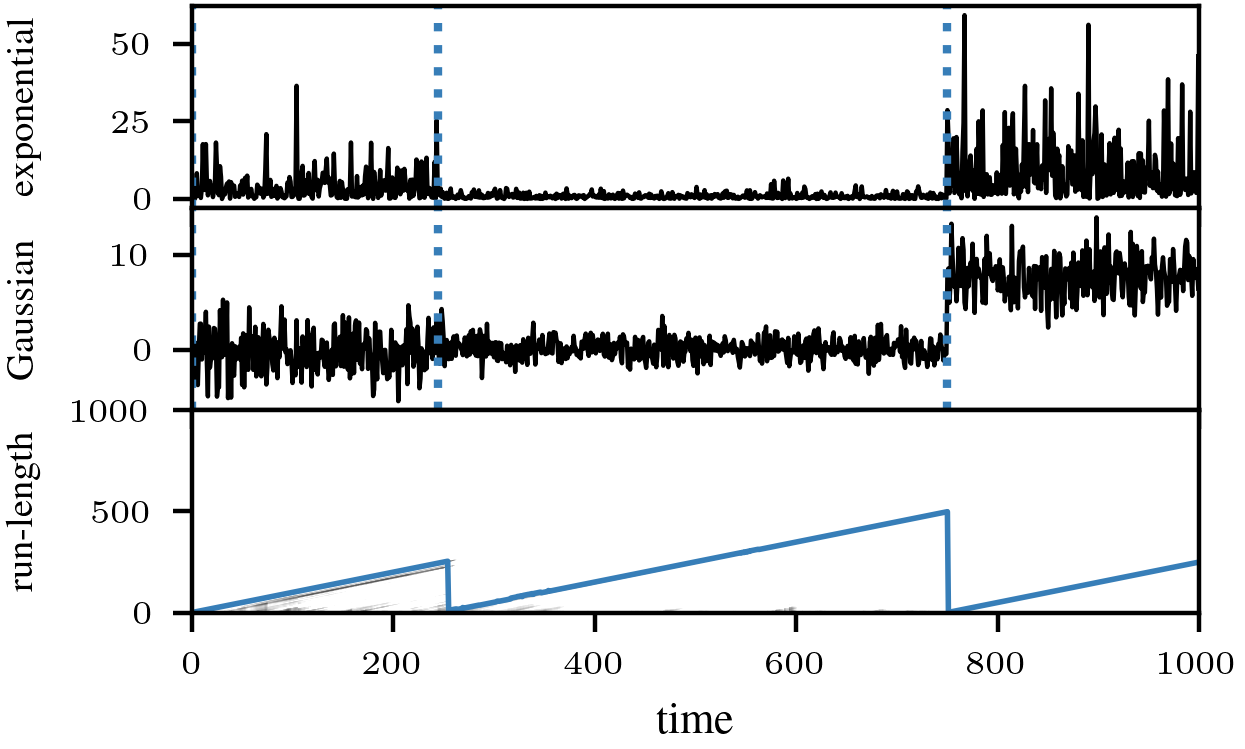}
    \vspace{-0.4cm}
    \caption{
    \textit{Multivariate synthetic example.}
    A 2-dimensional CP problem.
    For the chosen model, $\DSM_m$-BOCD is computationally efficient, but standard BOCD is computationally infeasible.
    The MAP segmentation is indicated by dashed {\textcolor{blue_plot}{\textbf{blue}}} lines, and the bottom panel shows the run-length distribution, with the most likely value in {\textcolor{blue_plot}{\textbf{blue}}}.}
    \vspace{-0.4cm} \label{fig:synthetic}
\end{figure}

\vspace{-3mm}

\paragraph{UK 10 year government bond yield.}
Finally, we run the $\DSM_m$-BOCD on the daily yield of 10 year UK government bonds from 2018 to 2022 (see \cref{fig:bond}). The data is publicly available via the Bank of England database.\footnote{\href{https://www.bankofengland.co.uk/boeapps/database/}{https://www.bankofengland.co.uk/boeapps/database/}}
Since the 10-year yield has been positive throughout history, we model it using the gamma distribution. 
As shown in \cref{fig:bond}, we detect changes in the yield curve that correspond to important political events in the UK.
This distribution leads to a  $\DSM_m$-posterior that is a  Gaussian truncated at zero. 
For standard Bayes, a conjugate prior exists, but it leads to a posterior with intractable normalisation constant. 
Like the multivariate synthetic data example, this constitutes another instance where $\DSM_m$-posteriors have better computational properties than standard Bayes.

%%%%%%%%%%%%%%%%%%%%%%%%%%%%%%%%%%%%%%%%%%%%%%%%%%%%%%%%%%%%%%%%%%%%%%%%%%%%%%%
%CONCLUSION
%%%%%%%%%%%%%%%%%%%%%%%%%%%%%%%%%%%%%%%%%%%%%%%%%%%%%%%%%%%%%%%%%%%%%%%%%%%%%%%
\section{Conclusion}

We proposed $\DSM_m$-BOCD, a new version of BOCD that is both \emph{robust to outliers and scalable}. 
The algorithm relies on a new generalised Bayesian inference scheme constructed with  diffusion score-matching.
These posteriors have closed form updates for models that are members of the exponential family, and provide robustness by appropriately tuning the diffusion matrix $m$.
For $T$ observations, $d$-dimensional data, and $p$ model parameters, the overall run time of the method is $\O(T(p^2+d^2))$, and we demonstrate that it is just as fast as standard BOCD.
By showcasing the various computational and inferential benefits of $\DSM_m$-BOCD on a range of examples, we demonstrate that it is a powerful and needed addition to the literature. 
In the future, we will also investigate the applicability of $\DSM_m$-BOCD to regression models.
This is not trivial: the regression setting changes both the definition of valid score matching losses, as well as how to show their robustness  \citep{xu2022generalized}.

$\DSM_m$-posteriors also are of independent interest for computational challenges in Bayesian inference: 
like the generalised posterior in \citet{matsubara2021robust} and \citet{matsubara2022generalised}, they can be computed even without access to the normalising constant of the likelihood.
This suggests that $\DSM_m$-posteriors should be studied more broadly as a potential competitor to other Bayesian methods for intractable likelihood problems.

\subsection*{Acknowledgements}
We would like to thank Ayush Bharti for feedback on a first draft of this paper.
JK was funded by  EPSRC grant EP/W005859/1. FXB was supported by the Lloyd’s Register Foundation Programme on Data-Centric Engineering and The Alan Turing Institute under  EPSRC grant EP/N510129/1.

%%%%%%%%%%%%%%%%%%%%%%%%%%%%%%%%%%%%%%%%%%%%%%%%%%%%%%%%%%%%%%%%%%%%%%%%%%%%%%%
%REFERENCES
%%%%%%%%%%%%%%%%%%%%%%%%%%%%%%%%%%%%%%%%%%%%%%%%%%%%%%%%%%%%%%%%%%%%%%%%%%%%%%%
\bibliography{Bibliography}
\bibliographystyle{icml2023}

%%%%%%%%%%%%%%%%%%%%%%%%%%%%%%%%%%%%%%%%%%%%%%%%%%%%%%%%%%%%%%%%%%%%%%%%%%%%%%%
%%%%%%%%%%%%%%%%%%%%%%%%%%%%%%%%%%%%%%%%%%%%%%%%%%%%%%%%%%%%%%%%%%%%%%%%%%%%%%%
% APPENDIX
%%%%%%%%%%%%%%%%%%%%%%%%%%%%%%%%%%%%%%%%%%%%%%%%%%%%%%%%%%%%%%%%%%%%%%%%%%%%%%%
%%%%%%%%%%%%%%%%%%%%%%%%%%%%%%%%%%%%%%%%%%%%%%%%%%%%%%%%%%%%%%%%%%%%%%%%%%%%%%%
\newpage
\appendix
\onecolumn

\vspace{10mm}
{
\begin{center}
\Large
    \textbf{Supplementary Materials}
\end{center}
}

In \Cref{appendix:background}, we provide mathematical background. In \Cref{appendix:proofs}, we present the proofs and derivations of all the theoretical results in our paper, while \Cref{appendix:expDetails} contains additional details regarding our experiments.
\section{Background}
\label{appendix:background}

Let $\nabla = (\partial/ \partial x_1,\ldots,\partial / \partial x_d)^\top$,  $f:\X\to\R^{d}$ and $g:\X\to\R^{d\times p}$; the divergence operator is define as follows: 
\begin{talign*}
    (\nabla\cdot f)(x) = \sum_{i=1}^{d} \dfrac{\partial f_{i}}{\partial x_{i}} (x),\qquad (\nabla\cdot g  )_{j}(x) = \sum_{i=1}^{d} \dfrac{\partial g_{ij}}{\partial x_{i}} (x),\quad \forall j \in \{1,...,p\}.
\end{talign*}

Expanding the term $\nabla\cdot mm^{\top}\nabla \log p_{\theta}(x)$ appearing as part of $d_m(\theta,x)$, we get
\begin{talign*}
    \nabla\cdot mm^{\top}\nabla \log p_{\theta}(x) &= \sum_{i=1}^{d} \dfrac{\partial}{\partial x_{i}}(mm^{\top}\nabla \log p_{\theta}(x))_{i}\\
    &= \sum_{i=1}^{d} \sum_{j=1}^{d}\dfrac{\partial}{\partial x_{i}}\left((mm^{\top})_{ij}(\nabla \log p_{\theta}(x))_{j}\right)\\
    &= \sum_{i=1}^{d} \sum_{j=1}^{d}\left(\dfrac{\partial}{\partial x_{i}}(mm^{\top})_{ij}\right)(\nabla \log p_{\theta}(x))_{j} +  \sum_{i=1}^{d} \sum_{j=1}^{d}(mm^{\top})_{ij}\left(\nabla^{2}\log p_{\theta}(x)\right)_{ij}\\
    &= \sum_{i=1}^{d} \sum_{j=1}^{d}\left(\dfrac{\partial}{\partial x_{i}}(mm^{\top})_{ij}\right)(\nabla \log p_{\theta}(x))_{j} +  \sum_{j=1}^{d} \left(mm^{\top}\nabla^{2}\log p_{\theta}(x)\right)_{jj}\\
    &= \sum_{j=1}^{d} \sum_{i=1}^{d}\left(\dfrac{\partial}{\partial x_{i}}(mm^{\top})_{ij}\right)(\nabla \log p_{\theta}(x))_{j} +  \Tr\left(mm^{\top}\nabla^{2}\log p_{\theta}(x)\right).
\end{talign*}
Where $\nabla^{2}$ is the Hessian. The expression in the last line is more straightforward to implement in practice and it is therefore the one we use in our code.

The term $\nu(x)$ in \cref{DSM-exponential} contains $(\nabla \cdot (mm^{\top}\nabla r)(x))$. The $j$-th index of this $p$-dimensional vector equals
\begin{talign*}
    (\nabla \cdot (mm^{\top}\nabla r)(x))_{j} &=  \sum_{i=1}^{d} \dfrac{\partial}{\partial x_{i}}(mm^{\top}\nabla r(x))_{ij}.
\end{talign*}

\section{Theoretical Results}\label{appendix:proofs}

In this section we present the derivation of the $\DSM_m$-posterior, along with the proof of its robustness. 

\subsection{Proof of \cref{DSM-exponential}}
\label{proof:DSM-exponential}
In this Subsection we present the proof of the main result of \cref{sec:conjugacy}: the conjugacy for exponential family models.

\textit{Proof.}
Let $p_{\theta}$ be an exponential family model. Then $\nabla\log{p_{\theta}}=\nabla r(x)^\top\eta(\theta)+\nabla b(x)$, and the DSM estimator has the following form:
\begin{talign*}
    \TDSM(\theta) = \frac{1}{T}\sum_{t=1}^{T}\underbrace{\|m^{\top}(\nabla r(x_{t})\eta(\theta)+\nabla b(x_{t}))\|_{2}^{2}}_{(1)}+2\underbrace{\nabla\cdot(mm^{\top}(\nabla r(x_{t})\eta(\theta)+\nabla b(x_{t})))}_{(2)}.
\end{talign*}

Let $\overset{+C}{=}$ indicate equality up to an additive term that does not depend on $\theta$.
\begin{talign*}
    (1)&=\eta(\theta)^{\top}\nabla r(x_{t})^{\top}mm^{\top} \nabla r(x_{t})\eta(\theta) + \nabla b(x_{t})^{\top}mm^{\top}\nabla b(x_{t}) + 2  \eta(\theta)^{\top}\nabla r(x_{t})^{\top}mm^{\top} \nabla b(x_{t})\\
    & \overset{+C}{=} \eta(\theta)^{\top}\nabla r(x_{t})^{\top}mm^{\top} \nabla r(x_{t})\eta(\theta) +  2  \eta(\theta)^{\top}\nabla r(x_{t})^{\top}mm^{\top} \nabla b(x_{t}),
\end{talign*}
and
\begin{talign*}
    (2)&= \nabla\cdot(mm^{\top}\nabla r(x_{t})\eta(\theta))+\nabla\cdot(mm^{\top}\nabla b(x_{t}))\\
    & \overset{+C}{=} \nabla\cdot(mm^{\top}\nabla r(x_{t})\eta(\theta))\\
    &= \eta(\theta)^{\top} (\nabla\cdot(mm^{\top}\nabla r(x_{t}))).
\end{talign*}
Therefore, $\TDSM(\theta) = \eta(\theta)^\top \Lambda_{T}\eta(\theta)+\eta(\theta)^\top \nu_{T}$ where
\begin{talign*}
    \Lambda_{T} &:= \dfrac{1}{T}\sum_{t=1}^{T}\nabla r(x_{t})^{\top}mm^{\top} \nabla r(x_{t}),\\
    \nu_{T} &:= \dfrac{2}{T}\sum_{t=1}^{T} \nabla r(x_{t})^{\top}mm^{\top} \nabla b(x_{t}) + \nabla\cdot(mm^{\top}\nabla r(x_{t})).
\end{talign*}
Now, assuming the prior has a p.d.f. $\pi$, the DSM-Bayes generalised posterior has a p.d.f.
\begin{talign*}
    \pi_{\omega}^{\DSM_{m}} \propto \pi(\theta) \exp{(-\omega T [\eta(\theta)^{\top} \Lambda_{T}\eta (\theta)+\eta (\theta)^{\top}\nu_{T}])}.
\end{talign*}
For $\eta(\theta)=\theta$, and the prior $\pi(\theta)\propto\exp{(-\frac{1}{2} (\theta-\mu)^{\top}\Sigma^{-1}(\theta-\mu))}$, we obtain the generalised posterior by completing the square as follows:
\begin{talign*}
    \pi_{\omega}^{\DSM_m}(\theta)&\propto\exp{(-\frac{1}{2} (\theta-\mu)^{\top}\Sigma^{-1}(\theta-\mu_{T}))} \exp{(-\omega T [\theta^{\top} \Lambda_{T}\theta+\theta^{\top}\nu_{T}])} \\
    & = \exp\left(-\dfrac{1}{2}\left(\theta^{\top}\Sigma^{-1}\theta- 2\theta^{\top}\Sigma^{-1}\mu+\mu^{\top}\Sigma^{-1}\mu+\theta^{\top}2\omega T \Lambda_{T}\theta +\theta^{\top}2\omega n \nu_{T}\right)\right)\\
    &\propto \exp\left(-\dfrac{1}{2}\left(\theta^{\top}(\Sigma^{-1}+2\omega T \Lambda_{T})\theta- 2\theta^{\top}(\Sigma^{-1}\mu - \omega T \nu_{T})\right)\right)\\
    &\propto\exp{\left(-\frac{1}{2} (\theta-\mu_{T})^{\top}\Sigma_{T}^{-1}(\theta-\mu_{T})\right)},
\end{talign*}
where 
\begin{talign*}
    \Sigma_{T}^{-1} &:= \Sigma^{-1}+2\omega T\Lambda_{T},\\
    \mu_{T} &:= \Sigma_{T} \left(\Sigma^{-1}\mu-\omega T \nu_{T}\right).
\end{talign*}
\qed
\subsection{Update Parameters for online DSM-Bayes}
\label{app:update}
In this Section we derive the efficient  parameter updates presented in \cref{sec:DSM-BOCD}.
To do so, we expand the expressions $\Lambda$ and $\nu$ as follows:
\begin{talign*}
    \Lambda_{T+1} &:= \dfrac{1}{T+1}\sum_{t=1}^{T+1}\nabla r(x_{t})^{\top}mm^{\top} \nabla r(x_{t})\\
    & = \dfrac{1}{T+1}\left(\sum_{t=1}^{T}\nabla r(x_{t})^{\top}mm^{\top} \nabla r(x_{t}) + \nabla r(x_{T+1})^{\top}mm^{\top} \nabla r(x_{T+1}) \right)\\
    &=\dfrac{1}{T+1}\left(n\Lambda_{T} + \nabla r(x_{T+1})^{\top}mm^{\top} \nabla r(x_{T+1})\right)\\
    \nu_{T+1} &:= \dfrac{2}{T+1}\sum_{t=1}^{T+1} \nabla r(x_{t})^{\top}mm^{\top} \nabla b(x_{t}) + \nabla\cdot(mm^{\top}\nabla r(x_{t}))\\
    &= \dfrac{1}{T+1}\left( T\nu_{T}+ 2(\nabla r(x_{T+1})^{\top}mm^{\top} \nabla b(x_{T+1}) + \nabla\cdot(mm^{\top}\nabla r(x_{T+1}))\right).
\end{talign*}
Now, assuming the prior has the conjugate form of \cref{DSM-exponential}, the  $\DSM_m$-posterior   is given by
\begin{talign*}
    \pi_{\omega}^{\DSM_M}(\theta)\propto\exp{\left(-\frac{1}{2} (\theta-\mu_{T})^{\top}\Sigma_{T}^{-1}(\theta-\mu_{T})\right)},
\end{talign*}
where 
\begin{talign*}
    \Sigma_{T}^{-1} &:= \Sigma^{-1}+2\omega n\Lambda_{T},\\
    \mu_{T} &:= \Sigma_{T} \left(\Sigma^{-1}\mu-\omega n \nu_{T}\right).
\end{talign*}
So the parameter updates for $\Sigma_{T}$ and $\mu_{T}$ as $T$ increases are:
\begin{talign*}
    \Sigma_{T+1}^{-1} &:= \Sigma^{-1}+2\omega (T+1)\Lambda_{T+1}\\
    & = \Sigma^{-1}+2\omega \left(n\Lambda_{T} +\nabla t(x_{T+1})^{\top}mm^{\top} \nabla t(x_{T+1})\right)\\
    &= \Sigma_{T}^{-1}+2\omega \nabla r(x_{T+1})^{\top}mm^{\top} \nabla r(x_{T+1})\\
    \mu_{T+1} &:= \Sigma_{T+1} \left(\Sigma^{-1}\mu-\omega (T+1) \nu_{T+1}\right)\\
    &= \Sigma_{T+1} \left(\Sigma^{-1}\mu-\omega \left( T \nu_{T}+ 2(\nabla r(x_{T+1})^{\top}mm^{\top} \nabla b(x_{T+1}) + \nabla\cdot(mm^{\top}\nabla r(x_{T+1}))\right)\right)\\
     &= \Sigma_{T+1} \left(\Sigma_{T}^{-1}\mu_{T} -2\omega\left( \nabla r(x_{T+1})^{\top}mm^{\top} \nabla b(x_{T+1}) + \nabla\cdot(mm^{\top}\nabla r(x_{T+1})\right)\right).
\end{talign*}
 obtaining the desired expression.
%%%%%%%%%%%%%%%%%%%%% GLOBAL ROBUSTNESS%%%%%%%%%%%%%%%%%%%%
\subsection{Global Bias-Robustness}

In this Subsection, we present the theory necessary to prove that the generalized posterior presented in \cref{sec:DSM-Bayes} is global bias-robust conditioned to the choice of $m$.

We first need to review a result from \citet{matsubara2022generalised} which states conditions for a discrepancy measure $\mathcal{D}(\theta)$ in order to prove that the corresponding generalised Bayes posterior $\pi_{\omega}^{\mathcal{D}}$ is globally robust to outliers. 
As in the original results, we will state our findings in terms of distributions $\P \in \mathcal{P}(\mathcal{X})$, where $\mathcal{P}(\mathcal{X})$ denotes the set of probability distributions on $\mathcal{X}$.
To this end, we will build on the notation introduced in \cref{sec:robustness}, and write
\begin{IEEEeqnarray}{rCl}
    \pi_{\omega}^{\mathcal{D}}(\theta | \P)\propto \pi(\theta) \exp\{-\omega T \cdot
\mathcal{D}(\theta; \P)\} \quad \quad \text{ for } \quad \quad  \mathcal{D}(\theta; \P) = \mathbb{E}_{X \sim \P}[d(\theta, X)].
\label{eq:gen-bayes-averaged-loss}
\end{IEEEeqnarray}
Here, the the discrepancy-based loss $\mathcal{D}(\theta; \P) = \mathbb{E}_{X \sim \P}[d(\theta, X)]$ allows us to recover theoretical posteriors based on averaging some kind of discrepancy $d:\Theta \times \mathcal{X} \to \mathbb{R}$ for any measure $\mathbb{P}$. 
This makes the results more general and more natural to derive. 
Note that all derived results apply to the $\DSM_m$-posterior computed from data points $x_{1:T}$, as we can recover it by considering $d=d_m$, and the corresponding empirical measure $\P_T = \frac{1}{T}\sum_{t=1}^T \delta_{x_t}$.
In particular, for $d=d_m$ we have that $\mathcal{D}(\theta; \P_T) = \widehat{\mathcal{D}}_m(\theta)$ so that $\pi_{\omega}^{\mathcal{D}}(\theta | \P_T) = \pi_{\omega}^{\mathcal{D}_m}(\theta | x_{1:T})$ as defined in \eqref{eq:DSM-bayes}.

The original work of \citet{matsubara2021robust} constructed a proof of robustness for posteriors that did not depend on an averaged loss $\mathcal{D}(\theta; \P) = \mathbb{E}_{X \sim \P}[d(\theta, X)]$, and so the conditions they derive do not exploit this averaged form.
Instead, they showed in Lemma 5 of their paper that global bias-robustness holds if
\begin{IEEEeqnarray}{rCl} \sup_{\theta\in\Theta} \sup_{y\in\X} \left|\frac{d}{d\varepsilon}\mathcal{D}(\theta;\P_{\varepsilon,y})|_{\varepsilon=0}(y,\theta,\P)\right|\pi(\theta) & < & \infty, \quad \quad \text{ and} \label{eq:matsubara-condition-1} \\ 
\int_\Theta \sup_{y\in\X} \left|\frac{d}{d\varepsilon}\mathcal{D}(\theta;\P_{\varepsilon,y})|_{\varepsilon=0}(y,\theta,\P)\right|\pi(\theta)d\theta & < & \infty.
\label{eq:matsubara-condition-2}
 \end{IEEEeqnarray}
Clearly, we can simplify this further because our loss is an average. 
As long as the function $d$ over which the loss is averaged is sufficiently regular, the below result shows that we obtain global bias-robustness.
%If we have a posterior whose loss is an expectation over $d_m$, then we can apply the below simplification of that result:

\begin{proposition}
\label{prop:DSM Global Bias-Robustness}
For each $\theta\in\Theta$. Suppose that $\pi$ is upper bounded over $\Theta$. If there exists a function $\gamma:\Theta\to\R$ such that:
 \begin{enumerate}
    \item $\sup_{y\in\X}|d(\theta,y)|\leq \gamma(\theta)$  ,     
    \item  $\sup_{\theta\in\Theta} \gamma(\theta)\pi(\theta)<\infty$ , and
    \item  $\int_\Theta \gamma(\theta)\pi(\theta)d\theta<\infty$.
 \end{enumerate}
 Then the posterior influence function $\operatorname{PIF}(y, \theta, \P)$ of $\pi_{\omega}^{\mathcal{D}}(\theta | \P)$ defined in  \eqref{eq:gen-bayes-averaged-loss} is bounded over both $ \theta \in \Theta$ and $y \in \mathcal{Y}$, so that the $\pi_{\omega}^{\mathcal{D}}(\theta | \P)$ is globally robust.
\end{proposition}

\begin{proof}

As outlined above, we simply have to show that the the above conditions suffice to guarantee  \eqref{eq:matsubara-condition-1} and \eqref{eq:matsubara-condition-2}. 
Rewriting the loss function related to the contamination model would be:
 \begin{talign*}
  \DSM(\theta;\P_{\varepsilon,y}) = \E_{x\sim\P_{\varepsilon,y}}[d(\theta,x)]
    =(1-\varepsilon)\E_{x\sim\P}[d(\theta,x)]+\varepsilon\E_{x\sim\delta_{y}}[d(\theta,x)].
\end{talign*}
Then, differentiating the last expression w.r.t. $\varepsilon$, and evaluating $\varepsilon=0$, we obtain:
\begin{talign*}
    \frac{d}{d\varepsilon}\DSM(\theta;\P_{\varepsilon,y})|_{\varepsilon=0}     =\E_{x\sim\P}[d(\theta,x)]+\E_{x\sim\delta_{y}}[d(\theta,x)]
\end{talign*}
Using Jensen's inequality, we bound the expression $|\frac{d}{d\varepsilon}\DSM(\theta;\P_{\varepsilon,y})|_{\varepsilon=0}|$ as follows:
\begin{talign*}
    |\frac{d}{d\varepsilon}\DSM(\theta;\P_{\varepsilon,y})|_{\varepsilon=0}|&\leq |\E_{x\sim\delta_{y}}[d(\theta,x)]|+|\E_{x\sim\P}[d(\theta,x)]|\\
    &\leq \E_{x\sim\delta_{y}}[|d(\theta,x)|]+\E_{x\sim\P}[|d(\theta,x)|]\\
    &= |d(\theta,y)|+\E_{x\sim\P}[|d(\theta,x)|]\\
    &\leq |d(\theta,y)|+\E_{x\sim\P}[\sup_{y\in\X}|d(\theta,y)|]\\
    &= |d(\theta,y)|+\sup_{y\in\X}|d(\theta,y)|,
\end{talign*}
and taking a supremum over $y$ we obtain the bound:
\begin{talign*}
\sup_{y\in\X}|\frac{d}{d\varepsilon}\DSM(\theta;\P_{\varepsilon,y})|\leq 2\sup_{y\in\X}|d(\theta,y)|\leq 2\gamma(\theta),
\end{talign*}
where the last inequality holds since $\gamma$ fulfil condition 1. Using this bound, we check the two conditions of \citet{matsubara2021robust}.
 \begin{enumerate}
     \item  $\sup_{\theta\in\Theta} \sup_{y\in\X} |\frac{d}{d\varepsilon}\DSM(\theta;\P_{\varepsilon,y})|\pi(\theta)\leq \sup_{\theta\in\Theta} 2\gamma(\theta)\pi(\theta)<\infty$ , and
     \item  $\int_\Theta \sup_{y\in\X} |\frac{d}{d\varepsilon}\DSM(\theta;\P_{\varepsilon,y})|\pi(\theta)d\theta\leq\int_\Theta 2\gamma(\theta)\pi(\theta)d\theta<\infty$,
 \end{enumerate}
where the last inequalities hold because $\gamma$ meets conditions 2 and 3. Therefore, by virtue of \citet{matsubara2021robust} the  posterior is globally bias-robust.
\end{proof}

 \subsection{Proof of \cref{Robust m}}
 In this Subsection we provide the proof of \cref{Robust m}.
 The strategy of the proof is simple: We show that $d = d_m$ admits a natural function $\gamma$ that satisfies the conditions of \cref{prop:DSM Global Bias-Robustness}.

 \begin{proof}
 From \cref{prop:DSM Global Bias-Robustness}, it is sufficient to find a function $\gamma$ such that:
\begin{talign*}
    \sup_{y\in\X}|\underbrace{\|m^{\top}(y)\nabla \log p_{\theta}(y)\|_{2}^{2}}_{(1)} + 2\underbrace{\nabla\cdot m(y)m^{\top}(y)\nabla\log p_{\theta}(y)}_{(2)} |\leq \gamma(\theta).  
 \end{talign*}
 Now, following from the form of $m$ in \cref{Robust m} and the fact that $p_{\theta}$ is an exponential family member as in \eqref{eq:exponential-family}, we have:
 \begin{talign*}
     (1) = \|m^{\top}(y)\nabla \log p_{\theta}(y)\|_{2}^{2} &= \sum_{i=1}^{d}(m^{\top}(y)\nabla \log p_{\theta}(y))_{i}^{2}
     = \sum_{i=1}^{d} \dfrac{(\nabla r(x)\theta)_{i}^{2}}{1+(\nabla r(x)\theta^{\star})_{i}^{2}}
     \leq \sum_{i=1}^{d} \dfrac{(\nabla r(x)\theta)_{i}^{2}}{(\nabla r(x)\theta^{\star})_{i}^{2}}.
 \end{talign*}
 Using the fact that $\|x\|_{2}^{2}\leq\|x\|_{1}^2\leq d\|x\|_{2}^{2}$ for $x\in\R^d$, we have:
 \begin{talign*}
     \sum_{i=1}^{d} \dfrac{(\nabla r(x)\theta)_{i}^{2}}{(\nabla r(x)\theta^{\star})_{i}^{2}} \leq \sum_{i=1}^{d} \dfrac{p\|\theta\|^2_2}{\|\theta^{\star}\|^2_2} = \dfrac{dp\|\theta\|^2_2}{\|\theta^{\star}\|^2_2} =: \gamma_{1}(\theta). 
 \end{talign*}
 For the second expression, we have:
 \begin{talign*}
    (2) =  |\nabla\cdot m(y)m^{\top}(y)\nabla\log p_{\theta}(y)| &= \left|\sum_{i=1}^{d} \dfrac{\partial}{\partial x_{t}} (m(y)m^{\top}(y)\nabla\log p_{\theta}(y))_{i}\right|\\
     &= \left|\sum_{i=1}^{d} \dfrac{\partial}{\partial x_{t}}\left(\dfrac{(\nabla r(x)\theta)_{i}}{1+(\nabla r(x)\theta^{\star})_{i}^{2}}\right)\right|\\
     &= \left|\sum_{i=1}^{d} \dfrac{(\nabla^2 r(x)\theta)_{ii}(1+(\nabla r(x)\theta^{\star})_{i}^{2}) - 2 (\nabla r(x)\theta^{\star})_{i}(\nabla r(x)\theta)_{i}(\nabla^2 r(x)\theta^{\star})_{ii}}{(1+(\nabla r(x)\theta^{\star})_{i}^{2})^2}\right|\\
     &\leq \sum_{i=1}^{d} \left|\dfrac{(\nabla^2 r(x)\theta)_{ii}}{1+(\nabla r(x)\theta^{\star})_{i}^{2})}\right| + 2 \left|\dfrac{(\nabla r(x)\theta^{\star})_{i}(\nabla r(x)\theta)_{i}(\nabla^2 r(x)\theta^{\star})_{ii}}{(1+(\nabla r(x)\theta^{\star})_{i}^{2})^2}\right|.
 \end{talign*}
For most distributions of interest, including Gaussians, exponentials, (inverse) Gamma, and Beta distributions, {\tiny$\left|\dfrac{(\nabla^2 r(x)\theta)_{ii}}{1+(\nabla r(x)\theta^{\star})_{i}^{2})}\right|$} is bounded for every $\theta \in \Theta$, then:
\begin{talign*}
    \sum_{i=1}^{d} \left|\dfrac{(\nabla^2 r(x)\theta)_{ii}}{1+(\nabla r(x)\theta^{\star})_{i}^{2}}\right|
    + 2 \left|\dfrac{(\nabla r(x)\theta^{\star})_{i}(\nabla r(x)\theta)_{i}(\nabla^2 r(x)\theta^{\star})_{ii}}{(1+(\nabla r(x)\theta^{\star})_{i}^{2})^2}\right| 
    & \leq dC(\theta)+2C(\theta)\sum_{i=1}^{d} \left|\dfrac{(\nabla r(x)\theta^{\star})_{i}(\nabla r(x)\theta)_{i}}{(1+(\nabla r(x)\theta^{\star})_{i}^{2})^2}\right| \\
    &\leq dC(\theta)(1+2d\dfrac{\|\theta\|^2_2}{\|\theta^{\star}\|^2_2}) =: \gamma_{2}(\theta).
\end{talign*}
 
 Defining $\gamma(\theta) := \gamma_1(\theta)+\gamma_2(\theta)$ we have :
 \begin{talign*}
    \sup_{y\in\X}\left|\|m^{\top}(y)\nabla \log p_{\theta}(y)\|_{2}^{2} + 2\nabla\cdot m(y)m^{\top}(y)\nabla\log p_{\theta}(y) \right|\leq \gamma(\theta).      
 \end{talign*}
 Now we are in a position to verify conditions (\Romannum{1}) and (\Romannum{2}) of \cref{prop:DSM Global Bias-Robustness}. Since $\gamma(\theta)$ is a polynomial function, $\pi(\theta)$ is a squared exponential prior, and the squared exponential has infinitely many moments,  it is clear that:
 \begin{talign*}
     \sup_{\theta\in\Theta} \pi(\theta)\gamma(\theta) &< \infty,\\
     \int_{\Theta} \pi(\theta)\gamma(\theta) d\theta & < \infty,
 \end{talign*}
 which completes the proof.
 \end{proof}
 %%%%%%%%%%%%%%%%%%%%%%% BOUNDARY CONDITIONS %%%%%%%%%%%%%%%%%%%%%%%%%%%%%%%%%%%%%%%%%%%
\subsection{Boundary and smoothness conditions}
\label{appendix:boundary}
In this subsection, we review the boundary and smoothness conditions discussed in \cref{sec:DSM-Bayes}, alongside the DSM extension to more general domains $\X$.

The smoothness conditions needed to get the expansion of $\DSM_{m}$ as in \cref{eq:DSM-expansion} for $\X=\R^{d}$ are the following:
\begin{lemma}
    If $p_{\theta}$ is twice-differentiable, and $p_{0}  m m^{\top} \nabla \log p_\theta , \nabla \cdot (p_{0} m m^{\top} \nabla \log p_{\theta}) \in L^1(\R^d)$, then we can rewrite $\DSM_{m}$ as in \cref{eq:DSM-expansion}.
\end{lemma}

\citet{Yu2019} extended the DSM to densities with non-negative support, i.e. $\X = \R^d_{\geq 0}$.

\begin{theorem}
\label{th:boundary1}
 Suppose that $\log p_{0}(x)$ and $m$ are continuously differentiable almost everywhere on $\R^d_+$ and  $\log p_{\theta}$ is 
twice continuously differentiable with respect to $x$ on $\R^d_+$. 
 Furthermore, we assume the boundary condition, 
\begin{talign*}
\lim_{|x^{(i)}|\to \infty} p_{0}(x) m_{ii}^{2}(x) \partial_{i}\log p_{\theta}(x)-\lim_{|x^{(i)}|\to 0+} p_{0}(x) m_{ii}^{2}(x) \partial_{i}\log p_{\theta}(x)=0, \forall i \in \{1,...,d\}, 
\end{talign*}
where $x_i$ is the i-dimension of $x$. Then, we can rewrite $\DSM_{m}$ as in \cref{eq:DSM-expansion}
\end{theorem}

Later, \citet{liu2022estimating} extended the DSM to densities with support in a Lipschitz Domain, which, intuitively speaking, are bounded connected open domains whose local boundary is a level set of some Lipschitz function. 

\begin{theorem}
    \label{th:boundary2}
	Assume $\X \subset \mathbb{R}^d$ is a Lipschitz domain.
	Suppose  $ p_{0}, \partial_{i}\log p_\theta\in H^1(\X)$ and
	that for any $z\in\partial{\X}$ it holds that 
	\begin{talign*}
    \lim_{x\!\rightarrow z} p_{0}(x) m_{ii}^2(x)\partial_{i}\log p_\theta(x)v_i(z) =0, \forall i \in \{1,...,d\}, 
	\end{talign*}
    where $x\rightarrow z$ takes any point sequence converging to $z\in\partial{\X}$ into account, $v = (v_1,...,v_d)$ is the unit outward normal vector on $\partial{\X}$, and $H^1(\X)$ is the Sobolev-Hilbert space.	Then, we can rewrite $\DSM_{m}$ as in \cref{eq:DSM-expansion}. 
\end{theorem}
The Sobolev-Hilbert space is defined as follows:
\begin{talign*}
H^{1}(\X) = \big\{ f\in L^2(\X)\,\big|\,\|f\|_{L^2(\X)}^2+\sum_{i=1}^{d} \|D_{i}f\|_{L^2(\X)}^2 <\infty\big\}, 
\end{talign*}
where $D_{i}$ is the weak derivative corresponding to $\partial_{i}$ and $\|f\|_{L^2(\X)} = \sqrt{\int_\X |f(x)|^2 dx}$. 
	
\paragraph{Gamma distribution}
Let $p_{\theta}$ be a gamma distribution, and $\X=\R_{+}$. Assume $p_0$ bounded from above and that $\log p_{0}(x)$ is continuously differentiable almost everywhere on $\R_+$. Then for $m$ as in \cref{Robust m}, we have
\begin{talign*}
    \lim_{|x|\to 0+} p_0(x)m^2(x)\partial\log p_\theta(x)  &= \lim_{|x|\to 0+}p_{0}(x)\dfrac{\nabla r(x)\theta + \nabla b(x)}{1+(\nabla r(x)\theta^{\star})_{i}^{2}} \\
    &= \lim_{|x|\to 0+}p_{0}(x)\dfrac{\frac{\theta_1}{x}-\theta_2}{1+(\frac{\theta_1^{\star}}{x}-\theta_2^{\star})^2} \\&=  \lim_{|x|\to 0+}p_{0}(x)\dfrac{x\theta_1-x^2\theta_2}{x+(\theta_1^{\star}-x\theta_2^{\star})^2} \\
    &= 0.
\end{talign*}
The last equality holds since $p_0$ is bounded. Now, for the second boundary condition:
\begin{talign*}
    \lim_{|x|\to \infty} p_0(x)m^2(x)\partial\log p_\theta(x)  &= \lim_{|x|\to \infty}p_{0}(x)\dfrac{\nabla r(x)\theta + \nabla b(x)}{1+(\nabla r(x)\theta^{\star})_{i}^{2}} = \lim_{|x|\to \infty}p_{0}(x)\dfrac{\frac{\theta_1}{x}-\theta_2}{1+(\frac{\theta_1^{\star}}{x}-\theta_2^{\star})^2} = 0.
\end{talign*}
The last equality holds since $p_0$ a density, therefore, $\lim_{|x|\to \infty}p_{0}(x)=0$. 
Then the $m$ proposed in \cref{Robust m} satisfies the boundary conditions in \cref{th:boundary1} for the gamma distribution. 

\paragraph{Exponential distribution}
Let $p_{\theta}$ be a exponential distribution, and $\X=\R_{+}$. Assume $p_0$ bounded, such that $\log p_{0}(x)$ is continuously differentiable almost everywhere on $\R_+$. Then for $m$ such as in \cref{Robust m} :
\begin{talign*}
    \lim_{|x|\to 0+} p_0(x)m^2(x)\partial\log p_\theta(x)  &= \lim_{|x|\to 0+}p_{0}(x)\dfrac{\nabla r(x)\theta + \nabla b(x)}{1+(\nabla r(x)\theta^{\star})_{i}^{2}}= \lim_{|x|\to 0+}p_{0}(x)= \dfrac{\theta}{1+\theta^{\star2}}\lim_{|x|\to 0+}p_{0}(x).
\end{talign*}
The term for the second limit would be similar:
\begin{talign*}
    \lim_{|x|\to \infty} p_0(x)m^2(x)\partial\log p_\theta(x) &= \dfrac{\theta}{1+\theta^{\star2}}\lim_{|x|\to \infty}p_{0}(x).
\end{talign*}
Therefore, the expression in \cref{th:boundary1} looks as follows:
\begin{talign*}
\lim_{|x^{(i)}|\to \infty} p_{0}(x) m_{ii}^{2}(x) \partial_{i}\log p_{\theta}(x)-\lim_{|x^{(i)}|\to 0+} p_{0}(x) m_{ii}^{2}(x) \partial_{i}\log p_{\theta}(x)=  \dfrac{\theta}{1+\theta^{\star2}}\left(\lim_{|x|\to \infty}p_{0}(x)-\lim_{|x|\to 0+}p_{0}(x)\right).
\end{talign*}
Then the $m$ proposed in \cref{Robust m} satisfies the boundary conditions in \cref{th:boundary1} if

$\lim_{|x|\to \infty}p_{0}(x) = \lim_{|x|\to 0+}p_{0}(x)$. 

\section{Additional Details on Numerical Experiments}
\label{appendix:expDetails} 
In this section we give additional details on the numerical experiments of \cref{sec:experiments}. We provide the exact prior and $\omega$ used in each experiment. Moreover, we present an extra numerical experiment to compare the computational complexity of standard BOCD and $\DSM_m$-BOCD in \cref{app:additional-computation-experiemnt}.
 
\subsection{Computational complexity}
\label{app:additional-computation-experiemnt}
The computational complexity of both standard BOCD and $\DSM_m$-BOCD  is linear in the number of data points. To verify that this theoretical complexity mirrors the practical computational overhead, we generate samples from a Gaussian distribution with 1 CP where the mean varies. We vary the sample size from $T=100$ up to $T=20000$. We fit a Gaussian distribution with the correct variance taken as fixed in both the $\DSM_{m}$-BOCD and the standard BOCD. As shown in \cref{fig:complexity_mean_T_ap}, both methods are equally fast for any number of observations.

Although the computational complexity of $\DSM_m$-BOCD  is linear in the data, it is quadratic in the dimension of the observations, in particular, is $\O(T(p^2+d^2))$. To observe this in practice, we consider the same settings as before, but now we fix the sample size to $T=100$ and vary the data dimensions from $d=1$ up to $d=500$. \cref{fig:complexity_mean_D} shows that both methods take practically the same time when $d$ is less than $100$.

\begin{figure}[h]
\centering
\begin{subfigure}{.48\textwidth}
  \centering
  \includegraphics[width=\linewidth]{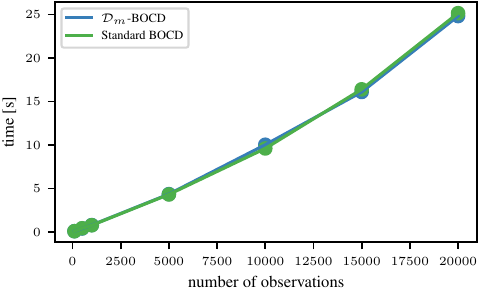}
  \caption{Overall time in seconds versus number of observations. We observe that both methods are equally fast for any number of observations.}
  \label{fig:complexity_mean_T_ap}
\end{subfigure}%
\hfill
\begin{subfigure}{.48\textwidth}
  \centering
  \includegraphics[width=\linewidth]{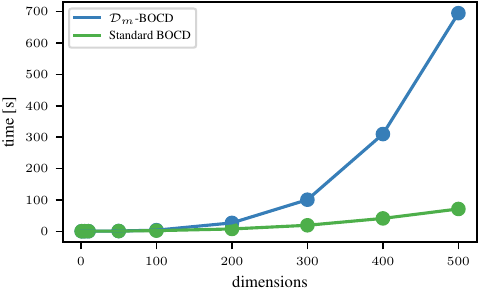}
  \caption{Overall time in seconds versus the dimension of the observations. We observe that both methods are practically equally fast when the dimension of the observations is less than $100$. }
  \label{fig:complexity_mean_D}
\end{subfigure}
\vspace{-0.3cm}
\caption{Comparison between $\DSM_m$-BOCD and the standard BOCD. {\textcolor{blue_plot}{\textbf{blue}}} line for $\DSM_m$-BOCD, and  {\textcolor{green_plot}{\textbf{green}}} line for standard BOCD.}
\label{fig:complexity_mean}
\end{figure}
The previous setting is where $\DSM_m$-BOCD clearly shows its advantage over the $\beta$-BOCD due to its completely closed form, making it nearly equivalent, in terms of computational overhead, to standard BOCD. In more complex settings, the posterior predictive may not be available in closed form; hence, we approximate it by sampling from $\pi_{\omega}^{\DSM_m}$. 
It might not be obvious how this is a significant advantage over the $\beta$-BOCD framework. %
To demonstrate that $\DSM_m$-BOCD is faster than $\beta$-BOCD, we generate samples from a Gaussian distribution with 1 CP, where the mean and the variance change. We vary the sample size from $T=100$ to $T=20000$ and fit a Gaussian distribution in the $\DSM_{m}$-BOCD, $\beta$-BOCD and the standard BOCD. In \cref{fig:complexity_T}, we observe that although the standard BOCD is faster than $\DSM_m$-BOCD, our method is considerably faster than the $\beta$-BOCD for any number of observations.
\begin{figure}[ht]
\centering
\includegraphics[width=.8\columnwidth]{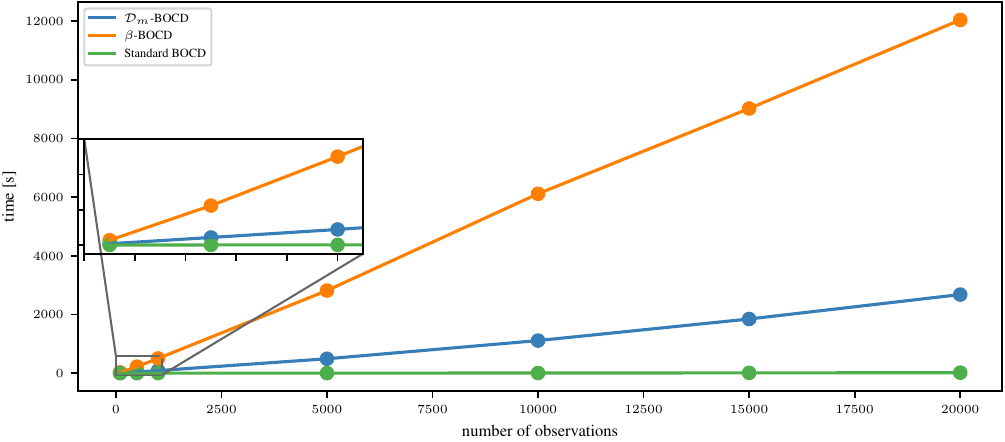}
\vspace{-0.4cm}
\caption{Overall time in seconds versus the number of observations. {\textcolor{blue_plot}{\textbf{blue}}} line for $\DSM_m$-BOCD, {\textcolor{orange_plot}{\textbf{orange}}} line for $\beta$-BOCD, and  {\textcolor{green_plot}{\textbf{green}}} line for standard BOCD. We observe that our method is faster than $\beta$-BOCD but is slower than the standard BOCD.}
\label{fig:complexity_T}
 \end{figure}
 
 \subsection{Varying $m$, $k$, and outliers intensity.}
 
We have demonstrated that our method is insensitive to the scale of the outliers. To do so, we compare the performance of the standard BOCD with the $\DSM_m$-BOCD in different contaminated datasets: while everything else stays the same, we scale the contamination points. We generate 600 samples with 6 CPs at $T\in\{200, 400\}$, and 3 contaminated point at $T\in\{100,300,500\}$. We contaminate the data by adding (or subtracting) 0, 5, 10, 20, and 40. For the $\DSM_m$-BOCD, we choose two different $m$ functions: $m$ as proposed to assure robustness, and $m=I_{d}$. In \cref{fig:varying_epsilon}, we observe that both Standard BOCD and  $\DSM_m$-BOCD with $m=I_d$ mistakenly label outliers as CPs, while $\DSM_m$-BOCD with $m$ robust is robust and identifies lasting changes.

\begin{figure}[h!]
\centering
\includegraphics[width=.9\columnwidth]{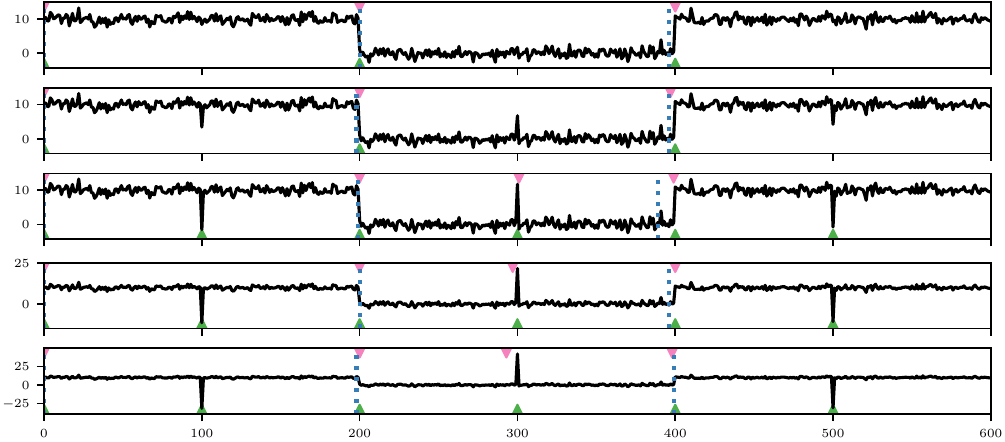}
\vspace{-0.3cm}
\caption{Contamination point scaled by 0, 5, 10, 20, and 40, from top to bottom. MAP segmentation indicated by {\textcolor{blue_plot}{\textbf{blue}}} dashed lines for $\DSM_m$-BOCD with $m$ robust, ${\color{pink_plot}\blacktriangledown}$ for $\DSM_m$-BOCD with $m=I_d$, and  ${\color{green_plot}\blacktriangle}$ for standard BOCD. 
    Both Standard BOCD and  $\DSM_m$-BOCD with $m=I_d$ mistakenly label outliers as CPs, while $\ DSM_m$-BOCD with $m$ robust is robust and identifies lasting changes.}
\label{fig:varying_epsilon}
 \end{figure}

 Finally, for contamination scaled by 10, we compared the performance and wall-clock time for $k\in\{1,50,600\}$. In \cref{fig:varying_k}, we observe that for $k=1$ the method cannot detect any changepoint, while for $k=50$ and $k=600$, the method successfully identifies the changepoints. However, in \cref{tab:varying_k} for $k=600$, the method takes 6x more time. As we discussed in the main paper, $k$ is a parameter in all variants of BOCD and will make all algorithms more expensive if increased.

 \begin{figure}[ht]
        \centering
        \includegraphics[width=.9\columnwidth]{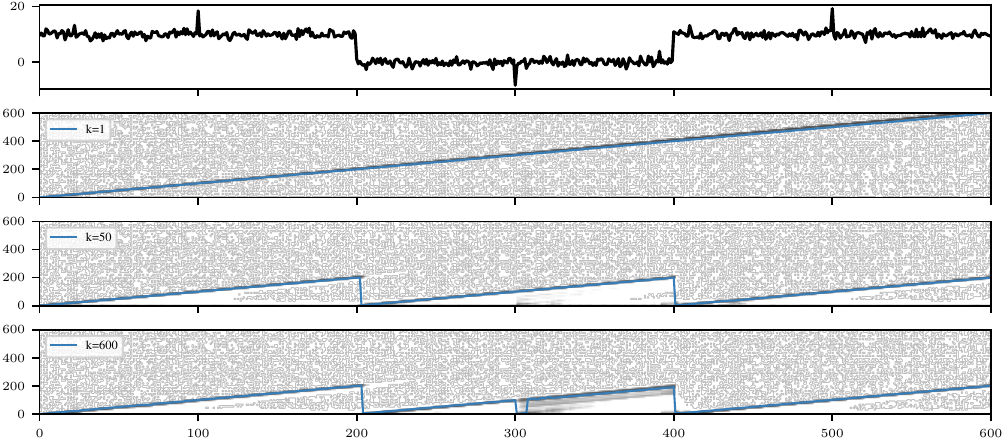}
        \caption{MAP segmentation indicated by {\textcolor{blue_plot}{\textbf{blue}}} dashed lines for $\DSM_m$-BOCD with $m$ robust. For $k = 1$ the method cannot detect any changepoint, while for $k = 50$ and $k = 600$, the method successfully
identifies the changepoints}
\label{fig:varying_k}
\end{figure}
        
    \begin{table}[ht]
    \centering
    \begin{tabular}{@{}ll@{}}
    \toprule
    k   & \multicolumn{1}{c}{time {[}s{]}} \\ \midrule
    1   & \textbf{1}                       \\
    50  & 38                               \\
    600 & 236                              \\ \bottomrule
    \end{tabular}
    \caption{Wall-clock time for varying $k$}
    \label{tab:varying_k}
    \end{table}

\newpage
\subsection{Accuracy and detection delay}
\label{app:accuracy}
We quantify the method’s performance advantage by comparing detection delay and accuracy on artificially generated data with outliers. In particular, We generate 600 samples with 2\% of outliers and 6 CPs; then, we report the positive predictive value (PPV), true positive rate (TPR), and the detection delays. These metrics are given by
\begin{equation*}
    \text{TPR} = \dfrac{\text{TP}}{\text{TP} + \text{FN}}, \quad \quad \text{PPV} = \dfrac{\text{TP}}{\text{TP} + \text{FP}},
\end{equation*}
where TP, FP, and FN are true positives, false positives, and false negatives, respectively. We say the method detects a TP changepoint when the changepoint detected for the method is in a small neighbourhood of the true CP. In order to measure how far the predicted CP is from the true CP, we measure the time difference between both and call it detection delay. For both PPV and TPR, the nearest to 1, the better. For delays, the lower, the better. We run the experiment 10 times, varying the position of the outliers. The following table shows the mean and standard deviation for each metric:

\begin{table}[h!]
\centering
\begin{tabular}{@{}llll@{}}
\toprule
Method        & PPV            & TPR            & Delays        \\ \midrule
DSM-BOCD      & \textbf{0.907$\pm$0.154} & \textbf{0.883$\pm$0.13} & 1.643$\pm$0.475         \\
Standard BOCD & 0.6$\pm$0.128            & 0.833$\pm$0.149          & \textbf{1.05$\pm$1.545} \\ \bottomrule
\end{tabular}
\caption{Performance indices-mean and standard deviation-using the positive predictive value (PPV), true positive rate (TPR), and the detection delays over 10 realisations. For PPV and TPR, the nearest to 1, the better. For delays, the lower, the better.}
\label{tab:accuracy-ap}
\end{table}

In \cref{tab:accuracy-ap}, we observe $\DSM_m$-BOCD has a significantly better performance with respect to PPV, meaning that the rate of false positives is lower than standard BOCD. This is a consequence of the robustness of our method. Moreover, we see a similar result concerning TPR, which measures the amount of changepoint not detected for the methods. Overall, this means that our method detects the same amount of TP as the standard BOCD while not detecting many FP, showing the strength of our method. Lastly, the detection delay shows that in spite of being robust to outliers, $\DSM_m$-BOCD does not cause any delay in the detection of CP.

\subsection{Twitter flash crash \& Cryptocrash}
In the Twitter flash crash experiment, we model the data with a Gaussian distribution, modelling its natural parameters. For $\DSM_m$-BOCD, we use a conjugate squared exponential prior r with parameters $\mu = (0, 1)^{\top}$, and $\Sigma$ a diagonal matrix so that $\diag(\Sigma) = (10, 1)$ for the natural parameters of said Gaussian. 
For standard BOCD, we use a Normal-inverse gamma prior with parameters $\mu_{0} = 0$, $\nu = 1$, $\alpha = 2$ and $\beta = 10$. 
We use the first 50 observations to select $\omega$ as in \cref{sec:DSM-BOCD}, with an  obtained value of $\omega^{\star} \approx 0.0001$

In the Cryptocrash experiment, we model the data with a multivariate Gaussian distribution modelling its natural parameters, and use conjugate squared exponential prior with parameters $\mu = (0, 1, 0, 1)^{\top}$, and $\Sigma$ diagonal matrix such that $\diag(\Sigma) = (2,1,2,1)$ for the $\DSM_m$-BOCD. For  standard BOCD, we model mean and variance instead, and use a normal-inverse-Wishart prior with parameters $\nu=0,\,\kappa= 1,\,\mu = (0, 0)$ and $\Psi$ diagonal matrix such that $\diag(\Psi) = (1,1)$. we manually fix $\omega = 0.01$. \cref{fig:tfx_full} shows the run-length posteriors and the Maximum A Posteriori (MAP) segmentations produced by each method.
Since the most likely run-lengths are virtually identical, the below plot also shows that in spite of being robust to outliers, $\DSM_m$-BOCD does not cause any delay in the detection of CPs.

 \begin{figure}[t!]
     \centering
     \includegraphics{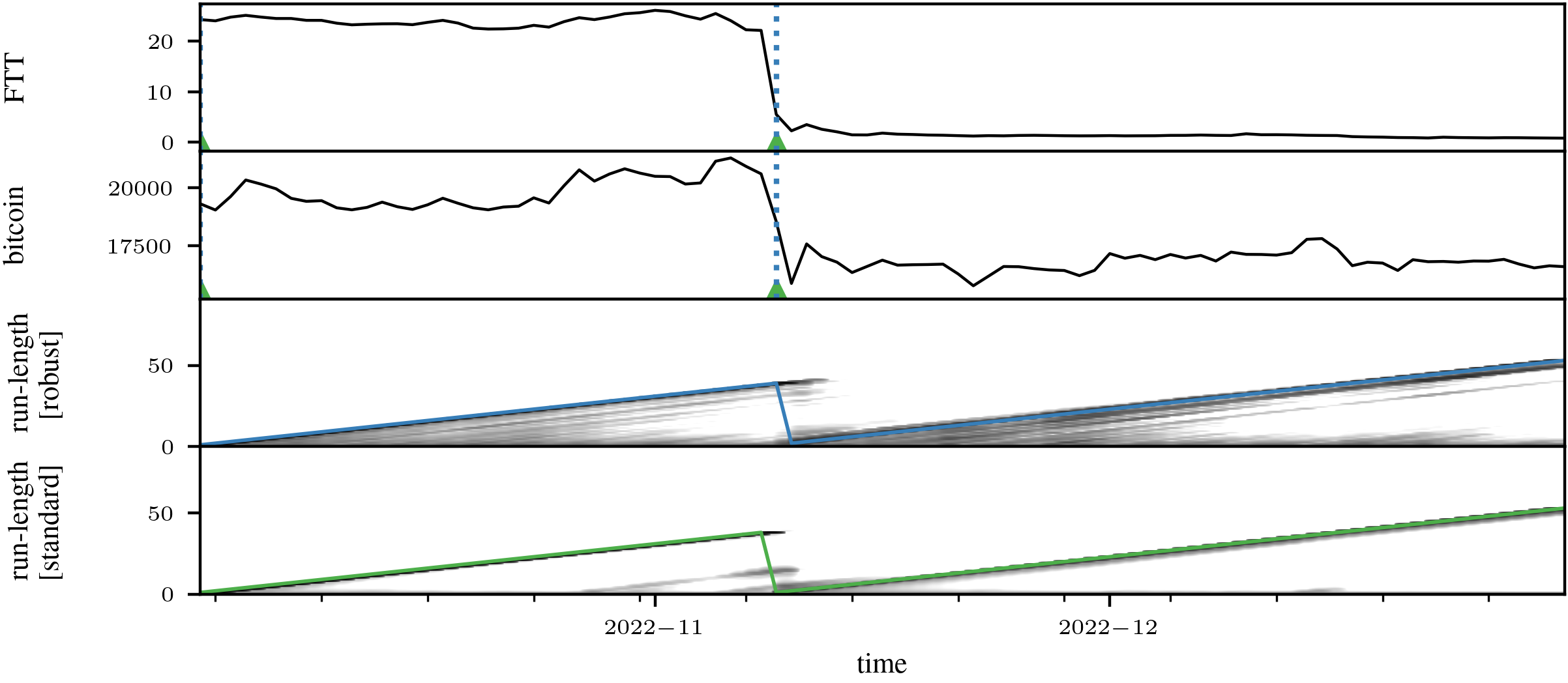}
      \caption{Maximum a posteriori (MAP) segmentation. In {\textcolor{blue_plot}{\textbf{blue}}} $\DSM$-BOCD segmentation using dashed lines. In \textcolor{green_plot}{\textbf{green}} Standard BOCD segmentation using ${\color{green_plot}\blacktriangle}$ marks. In addition we plot the run-length posteriors of robust $\DSM_m$-BOCD algorithm with most likely run-length in {\textcolor{blue_plot}{\textbf{blue}}} and of standard BOCD in \textcolor{green_plot}{\textbf{green}}. We observe that robustness does not lead to increased CP detection latency: both methods detect the CP at the same time.}
     \label{fig:tfx_full}
 \end{figure}
\subsection{Well-log}
We model the data with a Gaussian distribution, modelling its natural parameters and using a conjugate squared exponential prior with parameters $\mu = (0, 10)^{\top}$, and $\Sigma$ diagonal matrix such that $\diag(\Sigma) = (100, 100)$. \cref{fig:wellFull} shows the run-length and the segmentation of each method. We use the first 200 observations to select $\omega$ as in \cref{sec:DSM-BOCD}, with an  obtained value of $\omega^{\star} \approx 0.0004$
\subsection{Multivariate synthetic data}
We generate 1000 samples from a time series with CPs at t = 250, 750. Conditional
on the CPs, the data is generated independently from an exponential in the first a and Gaussian distribution in the second dimension. 
Since both dimensions are exponential family members, their joint distribution is too. On the natural parameters of this joint distribution, we place a conjugate squared exponential prior with parameters $\mu = (1,0,0.5)^{\top}$, and $\Sigma$ diagonal matrix such that $\diag(\Sigma) = (1, 1, 0.2)$. As we do not compare against BOCD on this data set, we simply fix $\omega^{\star} = 0.15$.

\subsection{UK 10 year government bond yield}
We model the data with a Gamma distribution and use a conjugate squared exponential prior with parameters $\mu = (0,1)^{\top}$, and $\Sigma$ diagonal matrix such that $\diag(\Sigma) = (50, 3)$ on its natural parameters. 
We use the first 100 observations to select $\omega$ as in \ref{sec:DSM-BOCD}. The obtained value is $\omega^{\star} \approx 0.05$.

\begin{figure*}[ht]
    \centering
    \includegraphics[width=0.85\columnwidth]{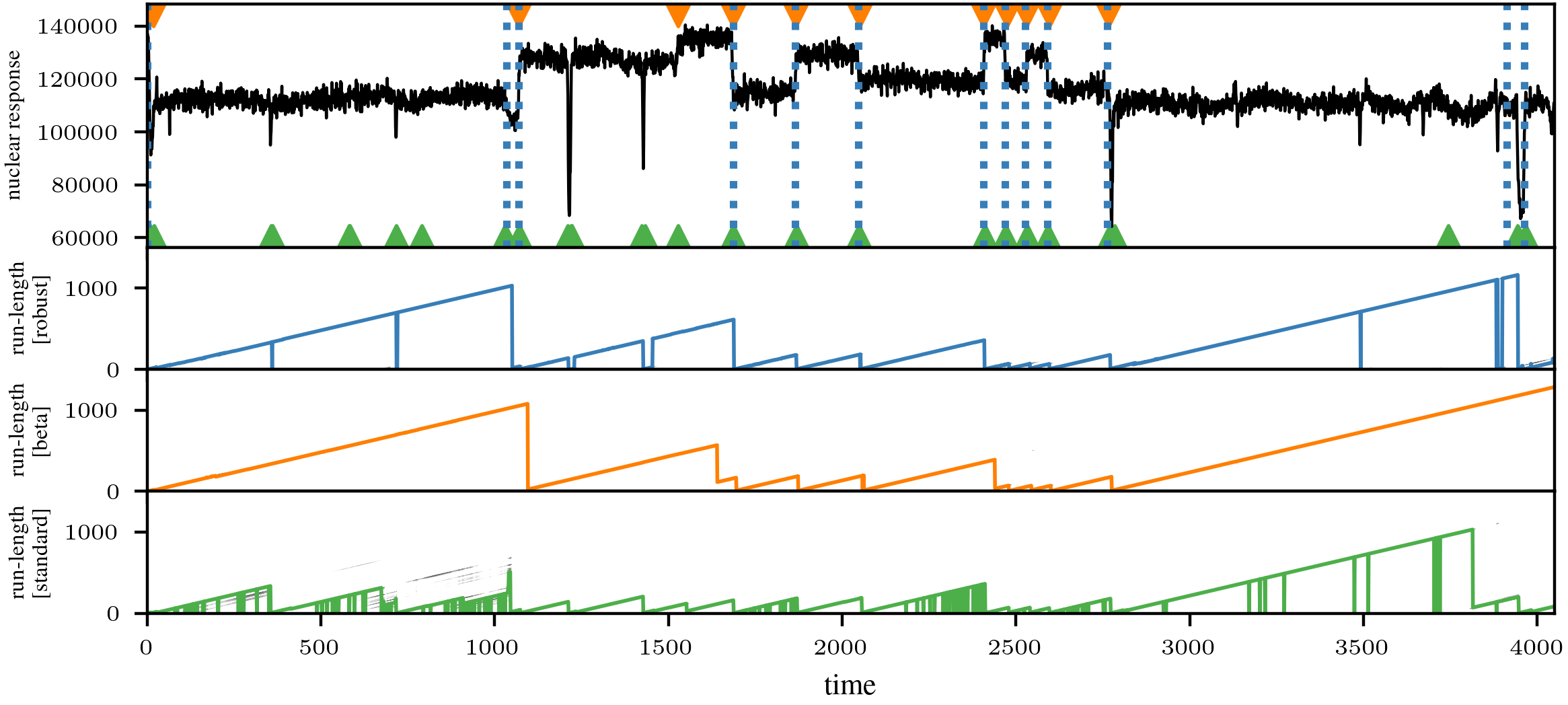}
    \caption{AP segmentation for the well-log indicated by {\textcolor{blue_plot}{\textbf{blue}}} dashed lines for $\DSM_m$-BOCD, ${\color{orange_plot}\blacktriangledown}$ for $\beta$-BOCD, and  ${\color{green_plot}\blacktriangle}$ for standard BOCD. In addition we plot the run-length posteriors of robust $\DSM_m$-BOCD algorithm with most likely run-length in {\textcolor{blue_plot}{\textbf{blue}}}, $\beta$-BOCD in {\textcolor{orange_plot}{\textbf{orange}}}, and of standard BOCD in \textcolor{green_plot}{\textbf{green}}. We observe that our method is more robust to outliers than the standard BOCD, but is more sensitive than $\beta$-BOCD.}
    \label{fig:wellFull}
\end{figure*}
%%%%%%%%%%%%%%%%%%%%%%%%%%%%%%%%%%%%%%%%%%%%%%%%%%%%%%%%%%%%%%%%%%%%%%%%%%%%%%%
%%%%%%%%%%%%%%%%%%%%%%%%%%%%%%%%%%%%%%%%%%%%%%%%%%%%%%%%%%%%%%%%%%%%%%%%%%%%%%%

\end{document}